\documentclass[11pt]{article}
\usepackage{amsthm}
\usepackage{amsmath}
\usepackage{amsfonts}
\usepackage{xcolor}
\usepackage{fullpage}
\usepackage{algorithm}
\usepackage{algorithmic} 
\usepackage{float}
\newfloat{algorithm}{t}{lop}

\newtheorem{theorem}{Theorem} 
\newtheorem{definition}[theorem]{Definition}

\newtheorem{lemma}[theorem]{Lemma}

\newtheorem{corollary}[theorem]{Corollary}
\newtheorem{remark}[theorem]{Remark}

\newcommand{\MRF}{I}
\newcommand{\CMRF}{\overline{I}}
\newcommand{\MRFMCMC}{\Gamma}
\newcommand{\ApproximateMRFMCMC}{\tilde{\Gamma}}
\newcommand{\HGL}{H}
\newcommand{\M}{M}
\newcommand{\N}{N}

\newcommand{\ignore}[1]{}

\def\shownotes{1}  \ifnum\shownotes=1
\newcommand{\authnote}[2]{[#1: #2]}
\else
\newcommand{\authnote}[2]{}
\fi

\newcommand{\note}[1]{\fbox{\parbox{16cm}{#1}}}

\title{Learning to Sample from Censored Markov Random Fields }
\author{Ankur Moitra \thanks{Department of Mathematics, Massachusetts Institute of Technology. Email: {\tt moitra@mit.edu}. This work was supported in part by NSF CAREER Award CCF-1453261, NSF Large CCF-1565235, a David and Lucile Packard Fellowship, an Alfred P. Sloan Fellowship and an ONR Young Investigator Award.} \\ MIT 
\and Elchanan Mossel \thanks{Department of Mathematics, Massachusetts Institute of Technology. Email: {\tt elmos@mit.edu}. This work was supported in part by NSF award DMS-1737944, Simons Investigator in Mathematics award (622132) and Vannevar Bush Faculty Fellowship ONR-N00014-20-1-2826}  \\ MIT \and Colin Sandon \thanks{Department of Mathematics, Massachusetts Institute of Technology. Email: {\tt csandon@mit.edu}. This work was supported in part by NSF award DMS-1737944 and Office of Naval Research Award N00014-17-1-2598}  \\ MIT}

\begin{document}

\maketitle

\begin{abstract}
We study learning Censor Markov Random Fields (abbreviated CMRFs). These are Markov Random Fields where some of the nodes are censored (not observed). We present an algorithm for learning high temperature CMRFs within $o(n)$ transportation distance. 
Crucially our algorithm makes no assumption about the structure of the graph or the number or location of the observed nodes. 
We obtain stronger results for high girth high temperature CMRFs as well as computational lower bounds indicating that our results can not be qualitatively improved. 

\end{abstract}

\setcounter{MaxMatrixCols}{20}



\newpage

\section{Introduction}
Graphical models provide a rich framework for describing high-dimensional distributions in terms of their
dependence structure. The problem of learning undirected graphical models from data has attracted much attention 
in statistics, artificial intelligence and theoretical computer science, including work 
on learning graphs of bounded degrees~\cite{AbKoNg:06,BrMoSl:08,Bresler:15,KlivansMeka:17} and under conditions of correlation decay~\cite{WaRaLa:06,BrMoSl:08}. 

It is natural to ask if learning can be also performed when some of the nodes are censored (i.e. we do not observe their values). Arguably, this question is more interesting than the question of learning graphical models with no censored nodes.
Note in particular, that this problem includes problems of learning generative deep networks \cite{hinton2009deep}. 
For models with hidden nodes there are many computational hardness results that are known to hold both in the worst-case~\cite{KlivansSherstov:09} and in average-case sense
~\cite{DanielyShalevShwartz:16,DanielyVardi:20} even for two layer networks. 
Even for tree models with hidden nodes, without additional assumption, learning is as hard as learning parity with noise~\cite{MosselRoch:05}. The results above either show that it is impossible to learn the function that relates the input to the output, or that it is impossible to learn (properly) the structure of the network. 
In a different direction~\cite{BoMoVa:08} considered computational questions related to the distributions generated 
by Markov Random Fields with hidden nodes. In particular, \cite{BoMoVa:08} shows that it is computationally hard to decide given samples coming from one of two MRFs whose distribution is far in total variation, which one they are coming from. 

In the positive direction, much is known for trees, where learning with no hidden nodes goes back to the 60s~\cite{ChowLiu:68}, and where learning tree models with hidden nodes, is a fundamental task in machine learning and Phylogenetic reconstruction.
Under non-degeneracy conditions, algorithms have been developed to learn the structure of 
such models~\cite{ErStSzWa:99a,Mossel:07} as well as their parameters~\cite{Chang:96,MosselRoch:05}. 
There are very few other models for which learning with hidden models has been established. 
These include sparse Ferromagnetic Ising models that are triangle-free and where hidden vertices are separated~\cite{BrMoSl:08}, 
treelike models with strong correlation decay~\cite{AnandkumarValluvan:13},  and Ferromagnetic Ising Boltzman machines~\cite{BrKoMo:19}.

\subsection{Our Contributions} 
Existing work predominantly considers structure learning, where the goal is to learn the structure of the underlying graph.
When there are hidden nodes, it is clear that some additional assumptions are needed in order to learn the structure of the graph. 
For example, in~\cite{BrMoSl:08,AnandkumarValluvan:13} different strong assumptions are made on the dispersion of hidden nodes 
inside the graph. Indeed, easy examples show that such assumptions are necessary. This can be seen for example where the set of observed nodes is very small (say empty or contains one element).  As a different elementary example, consider that in a censored Ising model, having a censored vertex with two edges has exactly the same effect on the probability distribution of the visible vertices as simply having an edge of the appropriate weight between those two vertices. 

A major contribution of the current paper is introducing a new learning model that can side-step the need to make any such assumption. The inputs to the model are independent samples from a censored Markov Random Field\footnote{ A censored Markov Random Field (CMRF from now on) is a Markov Random Field with some censored vertices whose values are omitted from samples as explained in definition \ref{CMRF}.} and the desired output is an efficient algorithm that approximately samples from the distribution. We call this new learning model, {\em learning to sample}. Not only does this allow us to provide learning algorithms without conditions on the dispersion of hidden nodes, it also provides a new conceptual framework for learning distributions.  

For example, one of the known criticism of structure learning MRFs, is that learning the structure is often not the end goal.
It is well known that sampling from a given MRF, even if it is a bounded degree Ising model is computationally hard unless $NP=RP$, 
see~\cite{Sly:10,SlySun:12}, the references within and follow up work. Thus the description of the distribution we get when we perform structure learning might not actually help us perform the down-stream tasks that we want. 

In our main technical results we show that it is possible to learn to sample CMRF in the high-temperature regime. 
The algorithm requires time $poly_{\epsilon}(n)$ to get within error $\epsilon n$ in earthmover distance. 
We obtain stronger results, where the error is $o(n)$ if the underlying CMRF has girth $\omega(1)$. 
We also provide a number of indications that to some extent, our results are the best one can hope for.
First, in the high temperature regime, we prove that if we can return a MRF whose distribution on the observed nodes is $o(n)$ in earth mover distance from the true distribution, then we could solve the parity with noise problem in polynomial time. 
We further show that for general low temperature models, assuming the existence of one way functions, learning to sample from a CMRF is hard.
Finally we show that the difficulty in learning to sample is indeed computational rather than statistical by providing an exponential time algorithm that learns to sample from a CRMFs using only polynomially many samples and without any restrictions.

\subsection{Terminology}
Write $[n]$ for $\{1,...,n\}$. Given $X\in\mathbb{R}^m, i \in [m], S \subset [m]$, we let 
\[
X_S := (X_j : j \in S), \quad, X_{-S} := X_{[m] \setminus S} = (X_j : j \notin S), \quad, X_{-i} := X_{-\{i\}}.
\]


A Markov random field is a set of variables associated with the vertices of a graph, such that each vertex is correlated with its neighbors, and independent of the rest of the graph conditioned on the values of its neighbors. In this paper we focus on the case where there are two possible values for each variable and all correlations between variables can be decomposed into pairwise correlations. 
As such, we recall the following definition of the {\em Ising model}.

\begin{definition}
Given $n>0$, an $n\times n$ symmetric matrix $\M$ with diagonal entries equal to $0$, and $\theta\in\mathbb{R}^n$, the {\em Ising model} with edge weights $\M$ and biases $\theta$ is the probability distribution over $\{-1,1\}^n$ such that if $X$ is drawn from this distribution and $x\in\{-1,1\}^n$ then
\[
\mathbb{P}[X=x]=\frac{\exp(\theta\cdot x+ \frac{1}{2} x\cdot \M x )}{\sum_{y \in\{-1,1\}^n} \exp(\theta\cdot y+\frac{1}{2}y \cdot \M y)}.
\]
We denote this probability distribution $\MRF_{(\M,\theta)}$.
\end{definition}

The function $Z_{\M,\theta} = \sum_{y \in\{-1,1\}^n} \exp(\theta\cdot y+\frac{1}{2}y \cdot \M y)$ is called the 
{\em partition function}. For a fixed $X\sim \MRF_{(\M,\theta)}$ and a vertex $i$, we say that $i$ has an {\em energy} of $-(\theta_i+ (\M X)_i)$. We call the model {\em $d$-sparse}  if in each row of $M$ has at most $d$ nonzero entries. 

\begin{remark} Almost all of our results generalize easily to other Markov Fields. To simplify the notation we state and prove the results for Ising models. A notable exception is our result for learning CMRFs in the correlation decay regime stated in Theorem~\ref{thm:high_girth}
\end{remark} 

In order to discuss a Markov random field where not all of the variables are observed, we define the following.

\begin{definition}\label{CMRF}
Given $n>0$, an $n\times n$ symmetric matrix $\M$ with diagonal entries equal to $0$, $\theta\in\mathbb{R}^n$, and $S\subseteq [n]$, the {\em censored markov random field} with edge weights $\M$, biases $\theta$, and visible vertices $S$ is the probability distribution of $X_S = (X_i : i \in S)$,  where $X\sim \MRF_{(\M,\theta)}$. We write this distribution $\CMRF_{S,(\M,\theta)}$ and call the vertices of $S^c$ its censored vertices.
\end{definition}

Our main interest is in learning CMRF in the sense that we can efficiently approximately sample from their distribution.

Thus, we introduce a {\em new} definition for learning CMRFs which requires the learning algorithm to be able to efficiently sample from a distribution that is close to the true MRF distribution. This definition is new even for MRFs.
Arguably, both for MRFs and for CMRFs, and in other settings of high dimensional distributions, it is the most natural definition of what it means to efficiently learn a distribution, particularly when learning the hidden dependencies is not in general possible.


\begin{definition}
Let $P$ be a probability distribution over some set $\Omega$, and $d$ be a distance on distributions. 
Let $A$ be a randomized algorithm that takes a possibly random number of samples in $\Omega$ and returns a value in $\Omega$. Given $x\in \Omega^\infty$, let $P(A,x)$ denote the probability distribution of the output of $A$ if the first sample it receives is $x_1$, the second sample it receives is $x_2$ and so on. We say that $A$ learns $P$ with distortion $\epsilon$ with respect to $d$ if
\[
E_{x\sim P^\infty}[d(P(A,x),P)] \leq \epsilon. 
\]
\end{definition}

\begin{remark}
This definition is subtle, and undergirds the entire paper. First, it is natural to wonder why a simple algorithm, such as outputting the first sample from the sequence $x$, does not accomplish the task above. The key point is the distance between the output of $A$ and $P$ is computed {\em after} we condition on the sequence. Thus an algorithm cannot merely parrot its input. The algorithm does not necessarily need to learn the parameters that define $P$ (such as the interactions and external fields in an Ising model) but it does need to have some other efficient mechanism for representing it so that it can generate genuinely new samples. 
\end{remark}

In other words, an algorithm learns a given probability distribution with low $d$ distortion, if for most random inputs, the probability distribution of the algorithm's output when it is given those samples is similar to the original probability distribution. Throughout this paper, we will be studying the difficulty of learning a CMRF with a low earthmover distortion.

\begin{definition}
Let $P,Q$ be two probability distributions on the same space $\Omega^n$. The {\em earth-mover} distance between $P,Q$, denoted $W(P,Q)$ is given by:
\[
W(P,Q) = \inf \{ E_{\mu}[\sum_{i=1}^n 1(x_i \neq y_i)] : \mu_1 = P, \mu_2 = Q \},
\]
where the infimum is taken over all couplings $\mu$ of $P,Q$. 
If $A$ is a randomized algorithm with output in $\Omega^n$, and $P(A,x)$ denotes the distribution of the output 
of $A$ on input $x$, we write $W_Q(A,Q)$ for $E[W(Q,P(A,X))]$, where $X = X_1,X_2,\ldots$, where 
$X_i$ are i.i.d. samples from $Q$.
\end{definition} 

Our goal in this paper is to find efficient algorithms $A$ for which $W_{\CMRF_{S,(\M,\theta)}}(A,\CMRF_{S,(\M,\theta)})$ is small. 

\subsection{Main Result}

A key insights of decades of research on Markov random field distinguishes between two types of regimes for MRFS. 
In the {\em low temperature} regime, 
\begin{enumerate}
\item 
the parameters ($\M_{v,u}$ in our case) are ``big" 
\item There is long range correlation 
\item Glauber dynamics converges slowly to the stationary distribution.
\end{enumerate} 
Recall also that more recent contribution shows that in the low temperature regime, there are graphs where approximate sampling is hard unless $NP=RP$. 

In the {\em high temperature regime}  
\begin{enumerate}
\item
The parameters ($\M_{v,u}$ in our case) are small 
\item There is no long range correlation 
\item  Glauber dynamics mixes rapidly. 
\end{enumerate}

Given the computational hardness of sampling from given low-temperature MRFs, we can only expect to learn to sample from high-temperature CMRF. This is formalized in Theorem~\ref{thm:1way} where we prove that hardness of learning CMRFs given one way functions. The reason this is not immediate from the hardness of sampling for low temperature models, is that the results proving hardness of sampling for MRFs are given as input the specification of the MRF, while we are given samples from the CMRF. 

In our main positive result we show that it is possible to learn (sufficiently) high temperature Ising models. 
\begin{theorem} \label{thm:high_temp_rec_intro}
For high temperature Ising models, for every $\epsilon > 0$, there is a polynomial time algorithm that learns
to sample CMRFs with earth-mover error $\epsilon n$ and polynomial time. 
\end{theorem}

See Theorem~\ref{thm:high_temp_rec} for the formal statement of the theorem. 
Interestingly, in Theorem~\ref{thm:girth} we improve the result above by obtaining a distribution whose 
 earth-mover distance $o(n)$ for CMRFs whose girth is $\omega(1)$ 

It is natural to ask if it is possible to learn with $o(n)$ error in earth-mover distance for general high-temperature 
CMRFs.  
Theorem~\ref{thm:parity} provides an indication that learning to sample with error $o(n)$ is computationally hard by showing that if there is an algorithms that returns a CMRFs whose output is within $o(n)$ of the given high temperature CMRF, then it is possible to learn sparse parities with noise. 

Finally, we show that the obstacle for learning to sample MRFs is computational and not information-theoretical as 
we show in Theorem~\ref{thm:exp_time} that given exponential time it is possible to learn to sample all CMRFs 
within total variation and earth-mover distance $o(1)$.

\section{Approximate Glauber Dynamics at High Temperatures}

The connection between temperature, decay of correlation and rate of converge for Glauber dynamics goes back to the work of Dobrushin and Dobrushin and Shlosman~\cite{Dobrushin:70,Kunsch:82,DobrushinShlosman:85}, see also~\cite{Martinelli:99,Weitz:04} among many other references. In this section we extend the results of Dobrushin and Dobrushin and Shlosman to the context of learning CMRFs. The main observation is that the arguments of Dobrushin and Dobrushin and Shlosman can be extended to much more general scenarios.
Since the proofs are similar to classical arguments in this theory we defer them to the appendix.

We will consider a natural generalization of Glauber dynamics as given in Algorithm~\ref{alg:general}. Usually Glauber dynamics is run in the setting where the parameters of the MRF are known, in which case we can set $$f_v = \mathbb{P}_{X\sim \MRF_{(\M,\theta)}}[X_{v}=1|X_{-v}=x_{-v}]$$
and after sufficiently many steps it would generate a sample from $I_{(M, \theta)}$. The key idea behind our approach is that while we will not be able to learn $M$ and $\theta$ in any reasonable sense, we can still approximate the above marginal distribution with another function $f_v$ that we can learn from samples, even if some of the variables are censored. In this section we show that a good enough approximation $f_v$ to the true marginal distribution is sufficient to guarantee that Algorithm~\ref{alg:general} outputs a sample from a distribution that is close to $I_{(M, \theta)}$ after a polynomial number of steps, provided that we are in a high-temperature setting.


\begin{algorithm}

\caption{$\ApproximateMRFMCMC$ - Approximate MRF MCMC} 
\label{alg:general}

\begin{algorithmic}
\STATE Input: A positive integer $T$, $x^{(0)}\in\{-1,1\}^n$, and a function $f$.

\STATE Output: An attempt at a sample from the desired distribution.

\STATE Set $x=x^{(0)}$.

\FOR{$0\le t< T$}
\STATE Randomly select $v \in [n]$.
\STATE Randomly select $p$ from the uniform distribution on $[0,1]$.
\IF{$p< f_v(x_{-v})$} 
\STATE set $x_{v}=1$ 
\ELSE 
\STATE set $x_{v}=-1$ 
\ENDIF
\ENDFOR

\RETURN $x$.

\end{algorithmic}
\end{algorithm}



\begin{algorithm}

\caption{$\MRFMCMC$ - Glauber Dynamics with Boundary Conditions} 
\label{alg:boundary}

\begin{algorithmic}

\STATE Input: The parameters $M,\theta$ of the desired Ising model, an initial value $x^{(0)}\in\{-1,1\}^{n}$, a positive integer $\tau$, and a set $W\subseteq [n]$.
\STATE Output: An attempt at a sample from the Ising model conditioned on the given values of the vertices not in $W$.

\STATE Set $x=x^{(0)}$.

\FOR{$0\le t< \tau$} 
\STATE Randomly select $v \in W$.
\STATE Randomly select $p$ from the uniform distribution on $[0,1]$.
\IF{$p< \mathbb{P}_{X\sim \MRF_{(\M,\theta)}}[X_{v}=-1|X_{-v}=x_{-v}]$} 
\STATE set $x_{v}=-1$. 
\ELSE
\STATE set $x_{v}=1$.
\ENDIF
\ENDFOR
\RETURN $x$.
\end{algorithmic}
\end{algorithm}


We can also specialize Algorithm~\ref{alg:general} in a different way to obtain a standard variant of Glauber dynamics on $W$ with ``boundary conditions" in $W^c$. See Algorithm~\ref{alg:boundary}. 
The fact that $v$ will always be in $W$ ensures that variables that are not in $W$ will maintain their original value. However, as $T$ goes to infinity, the probability distribution of the output of $\MRFMCMC(\M,\theta, x^{(0)},T, W)$ will get arbitrarily close to the probability distribution of $X\sim \MRF_{(\M,\theta)}$ conditioned on $X_{-W}=x^{(0)}_{-W}$. 
As mentioned earlier, in low temperatures it could take an exponentially large number of time steps for the probability distribution of $\MRFMCMC(\M,\theta, x^{(0)},T, W)$ to give a reasonable approximation of the desired probability distribution.
However, in high temperatures a polynomially large $\tau$ is sufficient. 
In order to formalize this, we will use a slight variant of the definitions used in \cite{Weitz:04}:

\begin{definition}
Given a vertex $v$, we define the {\em total influence {\bf on}} $v$ as:
\[
 \sum_{u\ne v} \max_{x\in\{-1,1\}^n} \left|\mathbb{P}[X_{v}=1|X_{u}=1,X_{-\{v,u\}}=x_{-\{v,u\}}]-\mathbb{P}[X_{v}=1|X_{u}=-1,X_{-\{v,u\}}=x_{-\{v,u\}}]\right|\]
We also define the {\em total influence {\bf of} $v$} as 
\[
\max_{x\in\{-1,1\}^n}
\sum_{u\ne v}  \left|\mathbb{P}[X_{u}=1|X_{v}=1,X_{-\{v,u\}}=x_{-\{v,u\}}]-\mathbb{P}[X_{u}=1|X_{v}=-1,X_{-\{v,u\}}=x_{-\{v,u\}}]\right|\]
\end{definition}

By the classical results of ~\cite{Dobrushin:70,Kunsch:82,DobrushinShlosman:85},
if either the total influence of every vertex is less than $1$ or the total influence on every vertex is less than $1$ then Algorithm~\ref{alg:boundary} attains a probability distribution similar to the true probability distribution of the MRF efficiently.

We note that for any vertex $v$, the total influence of $v$ and the total influence on $v$ are both less than or equal to
\[\sum_{u\ne v} \tanh(|\M_{v,u}|)\le \sum_{u\ne v} |\M_{v,u}| 
\]
Thus a natural measure of the inverse temperature is the quantity 
$\max_{v} \sum_{u} |\M_{v,u}|$. 

Our main result in this section is the following:

\begin{theorem}\label{robustSampling}
Let $0<\epsilon,\beta<1$ be constants. Let $n$ be a positive integer and $X\in \{-1,1\}^n$ be a random vector such that for all $v\in [n]$ and $x\in\{-1,1\}^n$,
\[\sum_{u\in [n]\backslash\{v\}} \left|\mathbb{P}[X_{u}=1|X_{-\{v,u\}}=x_{-\{v,u\}},X_v=1]-\mathbb{P}[X_{u}=1|X_{-\{v,u\}}=x_{-\{v,u\}},X_v=-1]\right|\le \beta\]
For each $v\in [n]$, let $f_v:\{-1,1\}^{n-1}\rightarrow [0,1]$ be a function such that
\[\sum_{v=1}^n \left|f_v(x_{-v})-\mathbb{P}[X_v=1|X_{-v}=x_{-v}]\right|\le \epsilon n\]
for all $x\in \{-1,1\}^{n}$. Then the probability distribution of the output of $\ApproximateMRFMCMC(f, x^{(0)},n\ln(n))$ is within an earthmover distance of $(2\epsilon/(1-\beta)+o(1))n$ of the probability distribution of $X$.
\end{theorem}




We will also use the following theorem and corollaries whose proofs are very similar to the proofs in~\cite{Dobrushin:70,Kunsch:82,DobrushinShlosman:85}  . 

\begin{theorem}\label{correlationDecay}
Let $b>0$ and $d$ and $n$ be positive integers such that $bd<1$. Let $\theta\in\mathbb{R}^n$ and $\M$ be an $n\times n$ symmetric matrix such that $\M_{i,i}=0$ for all $i$, $|\M_{i,j}|\le b$ for all $i$ and $j$, and for each $i$ there are at most $d$ values of $j$ for which $\M_{i,j}\ne 0$. Let $S\subseteq [n]$, $x\in\{-1,1\}^n$, $v, u\not\in S$. Also, let $\N \in\mathbb{R}^{n\times n}$ such that $\N_{i,j}=|\M_{i,j}|$ if $i,j\not\in S$ and $0$ otherwise. Then if $X\sim \MRF_{(A,\theta)}$, 
\[|\mathbb{P}\left[x_v=1|X_{S}=x_{S}, X_{u}=1\right]-\mathbb{P}\left[x_v=1|X_{S}=x_{S}, X_{u}=-1\right]|\le \sum_{k=0}^\infty \N^k_{v,u}\]
where we consider $ \N^0 = I$. 
\end{theorem}

\begin{corollary}\label{cor1}
Let $b>0$ and $d$ and $n$ be positive integers such that $bd<1$. Let $\theta\in\mathbb{R}^n$ and $\M$ be an $n\times n$ symmetric matrix such that $\M_{i,i}=0$ for all $i$, $|\M_{i,j}|\le b$ for all $i$ and $j$, and for each $i$ there are at most $d$ values of $j$ for which $\M_{i,j}\ne 0$. Let $S\subseteq [n]$, $x\in\{-1,1\}^n$, $v, u\not\in S$. Then
\[\sum_{u\in S: u\ne v} \left|\mathbb{P}[X_{u}=1|X_{S\backslash \{v,u\}}=x_{S\backslash \{v,u\}},X_v=1]-\mathbb{P}[X_{u}=1|X_{S\backslash \{v,u\}}=x_{S\backslash \{v,u\}},X_v=-1]\right|\le bd/(1-bd)\]
\end{corollary}

\begin{corollary}\label{cor2}
Let $b,r,\delta\ge 0$ and $d$ and $n$ be positive integers such that $bd<1$. Next, let $\theta\in\mathbb{R}^n$ and $\M$ be an $n\times n$ symmetric matrix such that $\M_{i,i}=0$ for all $i$, $|\M_{i,j}|\le b$ for all $i$ and $j$, and for each $i$ there are at most $d$ values of $j$ for which $\M_{i,j}\ne 0$. Next, let $S\subseteq [n]$, $v\in S$, $x\in \{-1,1\}^n$, and $G$ be the graph with adjacency matrix $\M$. Now, let $S'$ be a subset of $S$ containing all vertices that are connected to $v$ in $G$ by a path of length less than $r$ in which every edge has a weight with an absolute value of at least $\delta$, except for $v$ itself. Also, let $X\sim \CMRF_{S(\M,\theta)}$. Then
\[\left|\mathbb{P}[X_v=1|X_{S\backslash \{v\}}=x_{S\backslash \{v\}}]-\mathbb{P}[X_v=1|X_{S'}=x_{S'}]\right|\le [d\delta+(bd)^r]/(1-bd)\]
\end{corollary}

\section{Learning high temperature CMRFs}

\subsection{Probability approximation in the high temperature CMRF}

Now, consider trying to learn a high temperature censored markov random field. The main result of this section is that for every $\epsilon>0$ there is a polynomial time algorithm that uses samples drawn from a CMRF to learn a probability distribution that is within an earth mover distance of $\epsilon n$ from the original. More formally we prove: 

\begin{theorem} \label{thm:high_temp_rec}
Let $\epsilon,b,d>0$ such that $bd<1/2$. There exists $c,C>0$ and an algorithm $A'$ such that the following holds.  Let $n>0$ and $S\subseteq [n]$. Next, let $\theta\in[-b,b]^n$ and $\M$ be an $n\times n$ symmetric matrix such that $\M_{i,i}=0$ for all $i$, $|\M_{i,j}|\le b$ for all $i$ and $j$, and for each $i$ there are at most $d$ values of $j$ for which $\M_{i,j}\ne 0$. Then $A'$ runs in $O(n^c)$ time and satisfies 
\[
W_{\CMRF_{S(\M,\theta)}}(A',\CMRF_{S(\M,\theta)}) \leq \epsilon n+C.
\]

\end{theorem}



It is natural to ask if it is possible to obtain stronger results. For example:  can we get a $o(n)$ approximation in earth mover distance in polynomial time, 
or can we get a good approximation in total variation distance. The following result indicates that this might not be computationally possible. 
The result shows that it is computationally hard to find an CMRF whose distribution is close to the desired distribution.
Note that finding a CMRF whose distribution is close to the desired distribution is in principle harder than learning to sample. This is similar to the distinction between proper and improper learning. 
The hardness reduction is to the problem of learning sparse parities with noise, see~\cite{GrReVe:11}. 


\begin{theorem} \label{thm:parity}
Let $b,c,d,k>0$, such that $d\ge 8$ and $b(d+1)\le 1/2$, and let $\epsilon=2^{-(2k+8)}\delta_b^k/(k+1)$. Also, for $n>0$ and $S\subseteq[n]$, let $P_{S,n}$ be the probability distribution on $\{-1,1\}^n$ such that if $X\sim P_{S,n}$ and $x\in\{-1,1\}^n$ then $\mathbb{P}[X=x]=2^{-n}\left[1+\delta_b^k \prod_{i\in S} x_i\right]$,
where $\delta_b=\frac{\sinh^4(2b)}{2\cosh^4(2b)-\sinh^4(2b)}$.  

 Now, assume that there is an algorithm A such that for any CMRF with $n$ vertices, max degree at most $d$, and all edge weights and biases at most $b$, $A$ runs in $O(n^c)$ time, takes samples drawn from the CMRF and returns the parameters of a CMRF that satisfies the same criteria and is within an earth mover distance of $\epsilon n$ of the original CMRF with probability at least $1/2$. Then there is an algorithm $A'$ that runs in $O(n^c\log(n)+n^3)$ time such that for any $S\subseteq[n]$ with $|S|=2k+2$, this algorithm takes samples drawn from $P_{S,n}$ and returns $S$ with probability $1-o(1)$.
\end{theorem}

We note that~\cite{MosselRoch:05} reduced the problem of learning directed hidden Markov Models to the problem of learning parities with noise. Their result is stronger in that the parities are not spare, but it is weaker as they 
1. consider a directed model and 2. this model does not have correlation decay. 

However, it turns out that if we additionally assume that the graph has high girth, then it is possible to get within $o(n)$ earthmover distance:

\begin{theorem} \label{thm:girth}
Let $b$ and $d$ be positive constants such that $bd<1/2, n$ be a positive integer, $r=\omega(1)$, and $\theta\in[-b,b]^n$. Also, let $\M$ be an $n\times n$ symmetric matrix such that $\M_{i,i}=0$ for all $i$, $|\M_{i,j}|\le b$ for all $i$ and $j$, for each $i$ there are at most $d$ values of $j$ for which $\M_{i,j}\ne 0$, and every cycle in the weighted graph with adjacency matrix $\M$ has length greater than $r$. Finally, let $S\subseteq [n]$. 
Then there exists an algorithm $A$ that runs in polynomial time and such that
\[
W_{\CMRF_{S(\M,\theta)}}(A,\CMRF_{S(\M,\theta)}) \leq o(n). 
\]
\end{theorem}

\subsection{High Temperature Learning to Sample}

In order to prove Theorem~\ref{thm:high_temp_rec} we will first recall that for every $\epsilon'$, there exists $r$ such that the probability that a visible vertex is $1$ given an assignment of values to the {\em other visible} vertices is always within $\epsilon'$ of its probability of being $1$ given the values of the visible vertices that are within $r$ edges of it by corollary~\ref{cor1}. So, for every visible vertex in a CMRF, there is a small set of other visible vertices that can be used to predict its value with an accuracy that is nearly as good as the best that can be attained given the values of all other visible vertices. We next show that we can efficiently find such a set without knowing the parameters of the CMRF. Our plan for this is to simply try every small subset of visible vertices. Similarly to~\cite{BrMoSl:08}, we do this by checking every small subset against every other small subset. 
More formally, we prove the following:

\begin{lemma} \label{lem:2neighborhoods}
Let $\epsilon, b,d>0$ such that $bd<1$. There exists $c>0$ and an algorithm $L$ such that the following holds. Let $n>0$ and $S\subseteq [n]$. Next, let $\theta\in[-b,b]^n$ and $\M$ be an $n\times n$ symmetric matrix such that $\M_{i,i}=0$ for all $i$, $|\M_{i,j}|\le b$ for all $i$ and $j$, and for each $i$ there are at most $d$ values of $j$ for which $\M_{i,j}\ne 0$. Given the value of $n$ and the ability to query samples from $\CMRF_{S(\M,\theta)}$, $L$ runs in $O(n^c)$ time and with probability $1-o(1)$ it returns a collection of sets $S'_1,...,S'_n\subseteq S$ such that for every $v\in S$, $v\not\in S'_v$, $|S'_v|\le c$, and if $X\sim \CMRF_{S(\M,\theta)}$ then 
\[\left|\mathbb{P}[X_v=1|X_{S'_v}=x_{S'_v}]-\mathbb{P}[X_v=1|X_{S\backslash \{v\}}=x_{S\backslash \{v\}}]\right|\le \epsilon\]
for every $x$.
\end{lemma}

\begin{proof}
Given any $S^\star_1,S^\star_2\subseteq S$, let
\[D(S^\star_1,S^\star_2)=\max_{x\in\{-1,1\}^n} \left|\mathbb{P}[X_v=1|X_{S^\star_1}=x_{S^\star_1}]-\mathbb{P}[X_v=1|X_{S^\star_2}=x_{S^\star_2}]\right|\]

Now, let $r=\lceil \log((1-bd)\epsilon/4)/\log(bd)\rceil $, and pick $v\in S$. Next, let $S''$ be the set of all vertices in $S$ that are fewer than $r$ edges away from $v$ in the graph with adjacency matrix $\M$, except $v$ itself. $|S''|\le d^r$, and by corollary~\ref{cor1} we know that $D(S'',S\backslash\{v\})\le \epsilon/4$. In particular, that implies that $D(S'',S''\cup S^\star)\le \epsilon/4$ for any $S^\star\subseteq S\backslash\{v\}$
 because
\begin{align*}
&\min_{x': x'_{S''\cup S^\star}=x_{S''\cup S^\star}} \mathbb{P}[X_v=1|X_{S\backslash\{v\}}=x'_{S\backslash\{v\}}]\\
&\le \mathbb{P}[X_v=1|X_{S''\cup S^\star}=x_{S''\cup S^\star}]\\
&\le \max_{x': x'_{S''\cup S^\star}=x'_{S''\cup S^\star}} \mathbb{P}[X_v=1|X_{S\backslash\{v\}}=x'_{S\backslash\{v\}}]
\end{align*}

That in turn means that given $S^\star\subseteq S\backslash \{v\}$ such that $D(S^\star,S\backslash\{v\})> \epsilon$, it must be the case that 
\[D(S^\star,S''\cup S^\star)\ge D(S^\star,S\backslash\{v\})-D(S\backslash\{v\},S'')-D(S'',S''\cup S^\star)> \epsilon/2\]
 by the triangle inequality. So, in order to find a suitable value for $S'_v$, it suffices to find $S^\star_1\subseteq S\backslash \{v\}$ such that $|S^\star_1|\le d^r$ and $D(S^\star_1,S^\star_2\cup S^\star_1)\le \epsilon/2$ for all $S^\star_2\subseteq S\backslash\{v\}$ with $|S^\star_2|\le d^r$. We know that $D(S'',S''\cup S^\star)\le \epsilon/4$ for all $S^\star\subseteq S\backslash \{v\}$, so if we can find a way to estimate $D$ sufficiently accurately, we can find such a set by means of a brute force search.

Now, let $X_1,...,X_m\sim \CMRF_{S(\M,\theta)}$ where $m=\lceil \ln^2(n)\rceil$. Next, let $\widetilde{P}$ be the empirical probability distribution given by these samples and let $\widetilde{D}$ be the analogue of $D$ using $\widetilde{P}$ instead of $P$. For any $u\in S$ and $x\in\{-1,1\}^n$, it will be the case that $\mathbb{P}[X_{u}=x_{u}|X_{-u}=x_{-u}]\ge e^{-2b-2bd}/2$. So, for every $S^\star\subseteq S$, it will always be the case that $\mathbb{P}[X_{S^\star}=x_{S^\star}]\ge e^{-2b(1+d)|S^\star|}2^{-|S^\star|}$. That means that $|\widetilde{\mathbb{P}}[X_{S^\star}=x_{S^\star}]-\mathbb{P}[X_{S^\star}=x_{S^\star}]|\le \frac{\epsilon}{16} \mathbb{P}[X_{S^\star}=x_{S^\star}]$ for every $S^\star$ with $|S^\star|\le 2d^r+1$ with probability $1-o(1)$. If this holds, then for every $S^\star\subseteq S$ with cardinality at most $2d^r$, it will be the case that 
\begin{align*}
&|\widetilde{\mathbb{P}}[X_v=1|X_{S^\star}=x_{S^\star}]-\mathbb{P}[X_v=1|X_{S^\star}=x_{S^\star}]|\\
&=\frac{\left|\widetilde{\mathbb{P}}[X_v=1,X_{S^\star}=x_{S^\star}]\cdot \mathbb{P}[X_v=-1,X_{S^\star}=x_{S^\star}]-\widetilde{\mathbb{P}}[X_v=-1,X_{S^\star}=x_{S^\star}]\cdot \mathbb{P}[X_v=1,X_{S^\star}=x_{S^\star}]\right|}{\widetilde{\mathbb{P}}[X_{S^\star}=x_{S^\star}]\cdot \mathbb{P}[X_{S^\star}=x_{S^\star}]}\\
&\le \frac{\epsilon \mathbb{P}[X_v=1,X_{S^\star}=x_{S^\star}]\cdot \mathbb{P}[X_v=-1,X_{S^\star}=x_{S^\star}]}{8\widetilde{\mathbb{P}}[X_{S^\star}=x_{S^\star}]\cdot \mathbb{P}[X_{S^\star}=x_{S^\star}]}\\
&\le \frac{\epsilon \mathbb{P}[X_v=1,X_{S^\star}=x_{S^\star}]\cdot \mathbb{P}[X_v=-1,X_{S^\star}=x_{S^\star}]}{4\mathbb{P}^2[X_{S^\star}=x_{S^\star}]}\\
&\le \epsilon/16
\end{align*}
That in turn would imply that $|\widetilde{D}(S^\star_1,S^\star_2)-D(S^\star_1,S^\star_2)|\le \epsilon/8$ whenever $|S^\star_1|, |S^\star_2|\le d^r$. So, if this holds we can find a suitable value of $S'_v$ by using a brute force search to find $S^\star_1\subseteq S\backslash\{v\}$ with $|S^\star_1|\le d^r$ such that  $\widetilde{D}(S^\star_1,S^\star_1\cup S^\star_2)\le 3\epsilon/8$ for all $S^\star_2\subseteq S$ with $|S^\star_2|\le d^r$. For a given value of $(v, S^\star_1,S^\star_2)$ we can compute $\widetilde{D}(S^\star_1,S^\star_1\cup S^\star_2)$ in $O(\log^2(n))$ time, there are $n$ possible values of $v$, and for each $v$ there are only $O(n^{2d^r})$ pairs of sets we will need to check, so this runs in polynomial time.
\end{proof}



We can now prove Theorem~\ref{thm:high_temp_rec}. 

\begin{proof}
First of all, let $\epsilon'=\frac{\epsilon(1-2bd)}{4(1-bd)}$ and $X\sim \CMRF_{S(\M,\theta)}$. By the previous lemma, there exists a constant $c'$ and an algorithm $L$ that runs in $O(n^{c'})$ time that finds $S'_1,...,S'_n$ such that with probability $1-o(1)$, $v\not\in S'_v$, $|S'_v|\le c'$, and $|\mathbb{P}[X_v=1|X_{S'_v}=x_{S'_v}]-|\mathbb{P}[X_v=1|X_{S\backslash \{v\}}=x_{S\backslash \{v\}}]|\le\epsilon'$ for every $v\in S$ and $x\in\{-1,1\}^n$. Now, recall that if $|S'_v|\le c'$ and $x\in\{-1,1\}^n$, then $\mathbb{P}[X_v=x_v,X_{S'_v}=x_{S'_v}]\ge 2^{-c'-1}e^{-2b(d+1)(c'+1)}$. So, if we run $L$ and it outputs $S'$ of appropriate sizes, we only need a polynomial amount of additional time and samples to find functions $f_v:\{-1,1\}^{|S'_v|}\rightarrow [0,1]$ such that $|f_v(x_{S'_v})-\mathbb{P}[X_v=1|X_{S'_v}=x_{S'_v}]|\le\epsilon'$ for all $v\in S$ and $x\in\{-1,1\}^n$ with probability $1-o(1)$. In particular, if $L$ succeeds and this holds then
\[|f_v(x_{S'_v})-\mathbb{P}[X_v=1|X_{S\backslash\{v\}}=x_{S\backslash\{v\}}]|\le 2\epsilon'\]
for all $v\in S$ and $x\in\{-1,1\}^n$ by the triangle inequality. So, $L'$ will run $L$, attempt to find such $f$, and then run $\ApproximateMRFMCMC\left(f', \{1\}^{|S|},n\ln(n)\right)$ and return the result, where $f'_v(x)=f_v(x_{S'_v})$ for all $v$ and $x$. By theorem \ref{robustSampling} and corollary \ref{cor1}, the probability distribution of the output of $L'$ will be within an earthmover distance of $\left(4\epsilon'/\left(1-\frac{bd}{1-bd}\right)+o(1)\right)n=\epsilon n+o(n)$ of $\CMRF_{S(\M,\theta)}$, as desired.
\end{proof}

\subsection{Parity gadgets and the high temperature CMRF}

We now prove Theorem~\ref{thm:parity} which we restate. 

\begin{theorem}
Let $b,c,d,k>0$, such that $d\ge 8$ and $b(d+1)\le 1/2$, and let $\epsilon=2^{-(2k+8)}\delta_b^k/(k+1)$. Also, for $n>0$ and $S\subseteq[n]$, let $P_{S,n}$ be the probability distribution on $\{-1,1\}^n$ such that if $X\sim P_{S,n}$ and $x\in\{-1,1\}^n$ then $\mathbb{P}[X=x]=2^{-n}\left[1+\delta_b^k \prod_{i\in S} x_i\right]$,
where $\delta_b=\frac{\sinh^4(2b)}{2\cosh^4(2b)-\sinh^4(2b)}$.  

 Now, assume that there is an algorithm A such that for any CMRF with $n$ vertices, max degree at most $d$, and all edge weights and biases at most $b$, $A$ runs in $O(n^c)$ time, takes samples drawn from the CMRF and returns the parameters of a CMRF that satisfies the same criteria and is within an earth mover distance of $\epsilon n$ of the original CMRF with probability at least $1/2$. Then there is an algorithm $A'$ that runs in $O(n^c\log(n)+n^3)$ time such that for any $S\subseteq[n]$ with $|S|=2k+2$, this algorithm takes samples drawn from $P_{S,n}$ and returns $S$ with probability $1-o(1)$.
\end{theorem}

To prove the theorem, we will construct gadgets that force a set of vertices to have a given parity with probability nontrivially greater than $1/2$. That will allow us to set parameters for a CMRF such that the visible vertices are divided into small sets each of which takes on a specific parity with nontrivial probability and are otherwise independent. That will allow us to convert any proper learning algorithm for the high-temperature CMRF to an algorithm for learning $k$-parities with noise. That means that no efficient proper learning algorithm for the high temperature CMRF gets within distance $o(n)$ unless there is an algorithm that learns $k$-parities with noise in time $f(k)n^{O(1)}$. Our first step to proving this is to demonstrate a parity gadget. As such, we define the following.

\begin{definition}
For any $b>0$, $H_b$ is the weighted graph defined as follows. $H_b$ has $8$ vertices, $v_1$, $v_2$, $v_3$, $v_4$, $u_1$, $u_2$, $u_3$, and $u_4$. For each $i$, there is an edge of weight $b$ between $v_i$ and $u_i$, and for each $j\ne i$, there is an edge of weight $-b$ between $v_i$ and $u_j$.
\end{definition}

This functions as a parity gadget in the following sense
\begin{lemma}
Let $b>0$ and $X$ be drawn from an MRF corresponding to $H_b$, with all biases set to $0$. Then for each $x\in\{-1,1\}^4$,
\[
\mathbb{P}[X_{v_1,v_2,v_3,v_4}=x]=\frac{1}{16}(1+\delta_b)\prod_{i=1}^4 x_i
\]
\end{lemma}

\begin{proof}
First, let $Z$ be the partition function for this MRF. There are three cases we need to consider

Case 1: All entries in $x$ are equal. If $v_i$ takes on the same value for every $i$, then $u_j$ can either have energy $2b$ or energy $-2b$ for each $j$. So, $\mathbb{P}[X_{v_1,v_2,v_3,v_4}=x]=(e^{2b}+e^{-2b})^4/Z$.

Case 2: One entry in $x$ is different from the others. Assume without loss of generality that $x_1=1$ and $x_2=x_3=x_4=-1$. If the $v_i$ take on these values, $u_1$ can have energy $4b$ or $-4b$ and $u_j$ has energy $0$ for $j\ne 1$ regardless of its value. So, $\mathbb{P}[X_{v_1,v_2,v_3,v_4}=x]=8(e^{4b}+e^{-4b})/Z$.

Case 3: Two entries in $x$ are $1$ and the other two are $-1$. If the $v_i$ take on these values, then $u_j$ can have energy $2b$ or $-2b$ for each $j$. So, $\mathbb{P}[X_{v_1,v_2,v_3,v_4}=x]=(e^{2b}+e^{-2b})^4/Z$.
That means that 
\begin{eqnarray*}
Z&=&8(e^{2b}+e^{-2b})^4+8\cdot 8(e^{4b}+e^{-4b})\\
&=&128[\cosh^4(2b)+\cosh(4b)]\\
&=&128[\cosh^4(2b)+\sinh^2(2b)+\cosh^2(2b)]\\
&=&128[\cosh^4(2b)+\cosh^4(2b)-\sinh^4(2b)]\\
&=&128[2\cosh^4(2b)-\sinh^4(2b)]
\end{eqnarray*}
For $x$ with an even number of $1$'s, $\mathbb{P}[X_{v_1,v_2,v_3,v_4}=x]=16 \cosh^4(2b)/Z=\frac{1}{16}+8\sinh^4(2b)/Z$, and for $x$ with an odd number of $1$'s, $\mathbb{P}[X_{v_1,v_2,v_3,v_4}=x]=16 \cosh(4b)/Z=16(\cosh^4(2b)-\sinh^4(2b))/Z=\frac{1}{16}-8\sinh^4(2b)/Z$.
\end{proof}

The gadget above is a start, but it only works on $4$ vertices. We want ones that can effect arbitrarily large numbers. To do that, we take multiple copies of $H_b$ and combine them as follows.

\begin{definition}
For any $b>0$ and positive integer $k$, $H_{b[k]}$ is the weighted graph formed by taking $k$ copies of $H_b$ and then identifying the $v_4$ from the $i$th copy of $H_b$ with $v_1$ from the $(i+1)$th copy of $H_b$ for $1\le i<k$. We will refer to the copies of $v_1$, $v_2$, $v_3$, and $v_4$ that were not identified with other vertices as the focus vertices, and the other vertices in $H_{b[k]}$ as the background vertices.
\end{definition}

This functions as a parity gadget in the following sense.

\begin{lemma}
Let $b>0$, $k$ be a positive integer,  $X$ be drawn from an MRF corresponding to $H_{b[k]}$ with all biases set to $0$, and $x\in\{-1,1\}^{2k+2}$. The probability that the focus vertices take on the values given by $x$ is $2^{-(2k+2)}\left[1+\delta_b^k\prod x_i\right]$.
\end{lemma}

\begin{proof}
We induct on $k$. The previous lemma is the $k=1$ case. Now assume that this holds for $k-1$, and let $F$ be the set of focus vertices of $H_{b[k]}$. One can view $H_{b[k]}$ as $H_{b[k-1]}\cup H_b$ with two of the vertices identified; let $H'$ be the copy of $H_{b[k-1]}$, $H''$ be the copy of $H_b$, and $v$ be the vertex where they overlap. The restriction of $X$ to $H'$ and the restriction of $X$ to $H''$ are independent conditioned on $X_v$. Also, all of the biases are $0$, so the probability distribution of $X$ is symmetric over flipping all of the variables. So, $\mathbb{P}[X_v=1]=1/2$. That means that
\begin{eqnarray*}
\mathbb{P}[X_F=x] &=&\mathbb{P}[X_F=x,X_v=1]+\mathbb{P}[X_F=x,X_v=-1]\\
&=&\frac{\mathbb{P}[X_F=x|X_v=1]+\mathbb{P}[X_F=x|X_v=-1]}{2}\\
&=&\sum_{x'\in\{-1,1\}}\mathbb{P}[X_{F\cap H'}=x_{H'}|X_v=x']\cdot \mathbb{P}[X_{F\cap H''}=x_{H''}|X_v=x']/2\\
&=&\sum_{x'\in\{-1,1\}}2 \mathbb{P}[X_{F\cap H'}=x_{H'},X_v=x']\cdot \mathbb{P}[X_{F\cap H''}=x_{H''},X_v=x']\\
&=&\sum_{x'\in\{-1,1\}} 2^{-2k}\left[1+\delta_b^{k-1} x'\prod_{u\in F\cap H'} x_{u}\right]\cdot 2^{-4}\left[1+\delta_bx'\prod_{u\in F\cap H''} x_{u}\right]\cdot 2\\
&=&2^{-2k-3} \left[2+2\delta_b^k\prod_{u\in F} x_{u}\right]\\
&=&2^{-(2k+2)} \left[1+\delta_b^k\prod_{u\in F} x_{u}\right]\\
\end{eqnarray*}
This completes the proof. 
\end{proof}

This means that if we make a CMRF by taking some copies of $H_{b[k]}$, censoring their background vertices, and then adding in some extra vertices with no edges, we will get a CMRF that looks like a uniform distribution, but has noisy parities hidden in it. Unless we can find these parities, no CMRF that we construct in an attempt to duplicate this one will be quite right. More formally, we have the following.


\begin{proof}
In order to prove this, we will show that we can convert $P_{S,n}$ into an CMRF, and then extract $S$ from the field's edge weights. We can assume without loss of generality that $S$ is a random subset of $[n]$ with cardinality $2k+2$ because randomly permuting the indices converts the worst case version of the problem of determining $S$ given samples drawn from $P_{S,n}$ to the average case version of this problem. We will also refer to $S$ as $S_1$. Our algorithm will start by randomly selecting $m-1$ disjoint sets of size $2k+2$ in $[n]$, $S_2,...,S_m$ where $m=\lfloor n/(4k+4)\rfloor$. Now, let $P'$ be the probability distribution of the sequences produced by drawing $X\sim P_{S,n}$ and then replacing the elements indexed by $S_i$ with a tuple drawn from $P_{[2k+2],2k+2}$ for each $i>1$. Our algorithm can easily generate a sample from $P'$ by drawing a sample from $P_{S,n}$ and then making the necessary replacements. Also, with probability at least $2^{-(2k+2)}$, $S$ will be disjoint from $S_i$ for all $i>1$. If this holds, $X\sim P'$, and $x\in\{-1,1\}^n$ then $\mathbb{P}[X=x]=2^{-n}\prod_{i=1}^m \left[1+\delta_b^k\prod_{j\in S_i} x_j\right]$. 

In particular, $P'$ is the same as the CMRF where we have $n-m(2k+2)$ visible vertices with no edges and bias $0$, and $m$ copies of $H_{b[k]}$ with the $S_i$ as their sets of target vertices and their background vertices censored. So, we can run $A'$ on $P'$ to get a CMRF $M$ with degrees at most $d$, weights and biases at most $b$, and an earth mover distance of at most $2\epsilon n$ from $P'$ with probability at least $1/2$.  For large $n$, if this occurs then it must be the case that for at least half of $1\le i\le m$ it is the case that for all $x\in\{-1,1\}^{2k+2}$, and $X'\sim M$,
\[\Big|\mathbb{P}[X'_{S_i}=x]-2^{-(2k+2)}[1+\delta_b^k\prod x_j]\Big|\le 2^{-(2k+2)}\delta_b^k/2\]
In particular, by symmetry between the $S_i$, this must then hold for $S$ with probability at least $1/2$.

If this holds for $S$, then that means that for any $v,u\in S$, and conditioned on any fixed values of the entries of $X'$ corresponding to the other elements of $S$, the correlation between $X'_v$ and $X'_{u}$ has absolute value at least $\delta_b^k/2$. That implies that the distance between $v$ and $u$ on the graph given by $M$ is at most $2-k \log_2( \delta_b)$ due to the high temperature and low degree properties of $M$. There are at most $d^{4k+2-k(2k+1)\log_2(\delta_b)} n$ sets of $2k+2$ vertices in $M$ such that at least one of the vertices is visible and all of the vertices are within distance $2-k \log_2( \delta_b)$ of each other. So, it only takes $O(n^2)$ time to find them all. This collection of sets contains $S$ with a probability of at least $2^{-2k-4}$. So, if we carry out this entire procedure $\ln(n)$ times, then we will get a collection of $O(n\log(n))$ sets such that $S$ is in the collection with probability $1-o(1)$. Then, it only takes $O(n\log^3(n))$ time to draw $\ln^2(n)$ more samples from $P_{S,n}$ and check which of these set's variables have product $1$ most often. That will be $S$ with probability $1-o(1)$. This algorithm makes $O(\log(n))$ calls to $A$, and the rest of the algorithm runs in $O(n^3)$ time. So, this runs in $O(n^c\log(n)+n^3)$ time, as desired.
\end{proof}

\begin{remark}
If we removed the requirement that the CMRF output by A had to satisfy the same high temperature and low degree properties as the original graph, this argument would not work. The problem is that it would be possible to encode the samples in the censored portion of the graph, and then build gadgets to recover the $S_i$ from the samples. That would allow us to make the probability distribution of the visible vertices come out right without there being any obvious efficient way to determine the $S_i$ from the parameters of the CMRF. If A were required to learn to compute the probability distribution of a vertex given an assignement of values to the other vertices instead of giving the parameters of a CMRF then this argument would be mostly unchanged.
\end{remark}

\begin{remark}
This does not carry over to the low temperature case. With sufficiently loose restrictions on the edge weights it is possible to encode samples in the censored part of the graph, build gadgets that force a designated set of vertices to encode $S$, and then use more gadgets to impose the correct probability distribution on the visible vertices. Admittedly, this would probably require $\omega(n)$ censored vertices, so you would need a bigger graph.
\end{remark}

\subsection{High girth high temperature CMRFs}

At this point, it appears that we have the question of how accurately a polynomial-time algorithm can learn a high temperature CMRF essentially figured out. One possible follow up question is whether there is any additional restriction that we could put on the CMRF to make learning it easier. Our conclusion is that if the CMRF has high girth then there is a polynomial time algorithm that learns a probability distribution within an earthmover distance of $o(n)$ from it. The main advantage of the high-girth setting is that is much easier to identify observed nodes that are close to any fixed observed node. Thus, we can use a much more efficient procedure the one in Lemma~\ref{lem:2neighborhoods} to identify the vertices in $S$ that are in a neighborhood of the vertex. 

Our first step towards proving this will be to establish that in a high girth high temperature CMRF, any pair of nearby vertices with high-weight edges on the short path connecting them have a high correlation. More formally, we have the following.

\begin{lemma}\label{highGirthCorrelation}
Let $b\ge \delta>0$ and $d$, $r$, and $n$ be positive integers such that $bd<1$. Next, let $\theta\in[-b,b]^n$ and $\M$ be an $n\times n$ symmetric matrix such that $\M_{i,i}=0$ for all $i$, $|\M_{i,j}|\le b$ for all $i$ and $j$, for each $i$ there are at most $d$ values of $j$ for which $\M_{i,j}\ne 0$, and every cycle in the weighted graph with adjacency matrix $\M$ has length greater than $r$. Also, call this graph $G$. Next, let $v\ne u\in G$ such that there is a path of length $r'$ between them for some $r'<r/2$, and every edge in that path has a weight with an absolute value of at least $\delta$. After that, let $X\sim \MRF_{(\M,\theta)}$. Then
\[\left|\mathbb{P}[X_v=1|X_{u}=1]-\mathbb{P}[X_v=1|X_{u}=-1]\right|\ge e^{-4r'}\tanh^{r'}(\delta)-\frac{2(bd)^{r/2}}{1-bd}\]
\end{lemma}

\begin{proof}
Let $v_0,...,v_{r'}$ be the shortest path from $v$ to $u$, where $v_0=v$ and $v_{r'}=u$. Next, for each $0\le i$, let $S_i$ be the set of all vertices in $G$ that are exactly $i$ edges away from $v$. Also, let $S_{\ge i}=\cup_{j\ge i} S_j$ and $S_{< i}=\cup_{j< i} S_j$ for all $i$. Now, observe that the subgraph of $G$ induced by $S_{< r/2}$ is a tree, so we can compute the value of $\mathbb{P}\left[X_v=1\middle|X_{S_{\ge r/2}}=x_{S_{\ge r/2}},X_{u}=x_{u}\right]$ by means of belief propagation for any $x\in\{-1,1\}^n$. More precisely, if we set $N'(w)=\{w':\M_{w,w'}\ne 0, d(v,w')>d(v,w)\}$ and

\[p_{w}(x)=
\begin{cases}
x_{w} &\text{ if } w=u \text{ or } d(v,w)\ge r/2\\
\frac{\sum_{x'\in\{-1,1\}} x'e^{x' \theta_{w}}\prod_{w'\in N'(w)} \left[1+x' \tanh(\M_{w,w'})p_{w'}(x)\right]}{\sum_{x'\in\{-1,1\}}e^{x'\theta_{w}}\prod_{w'\in N'(w)} \left[1+x' \tanh(\M_{w,w'})p_{w'}(x)\right]} &\text{ otherwise }
\end{cases}\]

for all $w$ then $\mathbb{P}\left[X_v=1\middle|X_{S_{\ge r/2}}=x_{S_{\ge r/2}},X_{u}=x_{u}\right]=(p_v(x)+1)/2$ for all $x\in\{-1,1\}^n$. Now, let $x,x^\star\in\{-1,1\}^n$ such that $x_{u}=1$, $x^\star_{u}=-1$, and $x_{-u}=x^\star_{-u}$. For all $w$ not on the path between $v$ and $u$, it must be the case that $p_{w}(x)=p_{w}(x^\star)$ because either $d(v,w)\ge r/2$ or none of the elements of $N'(w)$ are on the path from $v$ to $u$ either, so this follows by induction on $\lceil r/2\rceil -d(v,w)$. Now, observe that $|p_{u}(x)-p_{u}(x^\star)|=2$. Next, let $Q_{w,w'}=\tanh(\M_{w,w'})p_{w'}(x)$ and $Q^\star_{w,w'}=\tanh(\M_{w,w'})p_{w'}(x^\star)$ for all $w$ and $w'$. The bounds on the entries of $\M$ imply that $|Q_{w,w'}|,|Q^\star_{w,w'}|\le \tanh(b)$ for all $w$ and $w'$. For any $0\le i<r'$,

\begin{align*}
&|p_{v_i}(x)-p_{v_i}(x^\star)|\\
&=\left|\frac{\sum_{x'\in\{-1,1\}} x'e^{x' \theta_{v_i}}\prod_{w'\in N'(v_i)} \left[1+x'Q_{v_i,w'}\right]}{\sum_{x'\in\{-1,1\}}e^{x'\theta_{v_i}}\prod_{w'\in N'(v_i)} \left[1+x' Q_{v_i,w'}\right]}-\frac{\sum_{x'\in\{-1,1\}} x'e^{x' \theta_{v_i}}\prod_{w'\in N'(v_i)} \left[1+x' Q^\star_{v_i,w'}\right]}{\sum_{x'\in\{-1,1\}}e^{x'\theta_{v_i}}\prod_{w'\in N'(v_i)} \left[1+x' Q^\star_{v_i,w'}\right]}\right|\\
&=\left|\frac{2\sum_{x'\in\{-1,1\}} x'\prod_{w'\in N'(v_i)} \left[1+x' Q_{v_i,w'}\right]\cdot \left[1-x' Q^\star_{v_i,w'}\right]}{\sum_{x'\in\{-1,1\}}e^{x'\theta_{v_i}}\prod_{w'\in N'(v_i)} \left[1+x'Q_{v_i,w'}\right]\cdot\sum_{x'\in\{-1,1\}}e^{x'\theta_{v_i}}\prod_{w'\in N'(v_i)} \left[1+x' Q^\star_{v_i,w'}\right]}\right| \\
&\ge \left|\frac{2\sum_{x'\in\{-1,1\}} x'\prod_{w'\in N'(v_i)} \left[1+x' Q_{v_i,w'}\right]\cdot \left[1-x' Q^\star_{v_i,w'}\right]}{4e^{2b}(1+\tanh(b))^{2d}}\right|\\
&= \left|\frac{2\sum_{x'\in\{-1,1\}} x'\left[1+x' Q_{v_i,v_{i+1}}\right]\cdot \left[1-x' Q^\star_{v_i,v_{i+1}}\right] \prod_{w'\in N'(v_i):w\ne v_{i+1}} \left[1-Q^2_{v_i,w'}\right]}{4e^{2b}(1+\tanh(b))^{2d}}\right|\\
&= \frac{4\left|Q_{v_i,v_{i+1}}-Q^\star_{v_i,v_{i+1}}\right|\prod_{w'\in N'(v_i):w\ne v_{i+1}} \left[1-Q^2_{v_i,w'}\right]}{4e^{2b}(1+\tanh(b))^{2d}}\\
&\ge\frac{\tanh(\delta)\left|p_{v_{i+1}}(x)-p_{v_{i+1}}(x^\star)\right|(1-\tanh^2(b))^d}{e^{2b}(1+\tanh(b))^{2d}}\\
&=\frac{\tanh(\delta)\left|p_{v_{i+1}}(x)-p_{v_{i+1}}(x^\star)\right|(1-\tanh(b))^d}{e^{2b}(1+\tanh(b))^{d}}\\
&=\frac{\tanh(\delta)\left|p_{v_{i+1}}(x)-p_{v_{i+1}}(x^\star)\right|}{e^{2b}e^{2bd}}\\
&\ge \frac{\tanh(\delta)\left|p_{v_{i+1}}(x)-p_{v_{i+1}}(x^\star)\right|}{e^4}\\
\end{align*}

Repeated application of this implies that $|p_v(x)-p_v(x^\star)|\ge 2e^{-4r'}\tanh^{r'}(\delta)$, and thus that 
\[ \left|\mathbb{P}\left[X_v=1\middle|X_{S_{\ge r/2}}=x_{S_{\ge r/2}},X_{u}=1\right]-\mathbb{P}\left[X_v=1\middle|X_{S_{\ge r/2}}=x_{S_{\ge r/2}},X_{u}=-1\right]\right|\ge e^{-4r'}\tanh^{r'}(\delta)\]

Now, recall that
\[\left|\mathbb{P}\left[X_v=1\middle|X_{S_{\ge r/2}}=x_{S_{\ge r/2}},X_{u}=1\right]-\mathbb{P}\left[X_v=1\middle|X_{u}=1\right]\right|\le \frac{(bd)^{r/2}}{1-bd}\]
and
\[\left|\mathbb{P}\left[X_v=1\middle|X_{S_{\ge r/2}}=x_{S_{\ge r/2}},X_{u}=-1\right]-\mathbb{P}\left[X_v=1\middle|X_{u}=-1\right]\right|\le \frac{(bd)^{r/2}}{1-bd}\]
Therefore,
\[\left|\mathbb{P}\left[X_v=1\middle|X_{u}=1\right]-\mathbb{P}\left[X_v=1\middle|X_{u}=-1\right]\right|\ge e^{-4r'}\tanh^{r'}(\delta)-\frac{2(bd)^{r/2}}{1-bd}\]
\end{proof}

In particular, this means that given a vertex $v$ in a high temperature high girth CMRF, we can find all visible vertices that are connected to $v$ by short paths with large edge weights by finding all vertices that are highly correlated with $v$. So, we can generate samples from a probability distribution approximating a given high temperature high girth CMRF by drawing some samples, using them to find all pairs of vertices with sufficiently high correlations, defining functions to estimate the probability that those vertices are $1$ based on the values of the vertices that are highly correlated with them, and then running $\ApproximateMRFMCMC$. More formally, we can use the following algorithm.

\begin{algorithm}

\caption{$\HGL(n, m, S, \rho_0, P_X)$ - High Girth Sampling}

\begin{algorithmic}

\STATE Input: The number of vertices $n$, a positive integer $m$, the set of visible vertices $S$, a positive real number $\rho_0$, and a sampling oracle for the CMRF, $P_X$.

\STATE Output: An attempt at a fresh sample from the CMRF.

\FOR{$0\le i<m$} 
\STATE draw $X^{(i)}\sim P_X$.
\ENDFOR

\FOR{$v,u\in S$} 
\STATE set $\rho_{v,u}$ to the empirical correlation between $X_v$ and $X_{u}$ on these samples.
\ENDFOR

\FOR{$v\in S$}
\STATE $S_v := \{u\in S: u\ne v, |\rho_{v,u}|\ge \rho_0\}$.
\ENDFOR 

\FOR{$v\in S$} 
\STATE define $f_v:\{-1,1\}^{n-1}\rightarrow[0,1]$ so that $f_v(x)=|\{i: X^{(i)}_v=1, X^{(i)}_{S_v}=x_{S_v}\}|/|\{i: X^{(i)}_{S_v}=x_{S_v}\}|$.
\ENDFOR

\RETURN $\ApproximateMRFMCMC(f, \{1\}^n,n\ln(n))$.

\end{algorithmic}
\end{algorithm}

This can learn high temperature high girth CMRFs in the following sense.

\begin{theorem} \label{thm:high_girth}
Let $b$ and $d$ be positive constants such that $bd<1/2$. Next, let $n$ be a positive integer, $r=\omega(1)$, and $\theta\in[-b,b]^n$. Also, let $\M$ be an $n\times n$ symmetric matrix such that $\M_{i,i}=0$ for all $i$, $|\M_{i,j}|\le b$ for all $i$ and $j$, for each $i$ there are at most $d$ values of $j$ for which $A_{i,j}\ne 0$, and every cycle in the weighted graph with adjacency matrix $\M$ has length greater than $r$. Finally, let $S\subseteq [n]$. Then the probability distribution of $\HGL\left(n, n , S, 1/\ln(\ln(n)), \CMRF_{S,(\M,\theta)}\right)$ is within an earthmover distance of $o(n)$ of $\CMRF_{S,(\M,\theta)}$.
\end{theorem}

\begin{proof}
First, let $X\sim \CMRF_{S,(\M,\theta)}$, and observe that $e^{-2bd-2b}/2\le \mathbb{P}[X_v=1]\le 1-e^{-2bd-2b}/2$ for any $v\in[n]$. There are $O(n^2)$ pairs of vertices that $\HGL$ attempts to estimate the correlations between, and it has $n$ samples, so with probability $1-o(1)$ all of its estimates are within $1/2\ln(\ln(n))$ of the true values. If this holds then for each $v$, every element of $S_v$ will be within a distance of $O(\log(\log(\log(n))))$ of $v$ by the corollary to lemma \ref{correlationDecay}, which implies that $|S_v|=o(\log(n))$. That in turn implies that with probability $1-o(1)$, for each $v$ there will be at least $\sqrt{n}$ samples in which the vertices in $S_v$ take on each possible value. So, 
\[|f_v(x)-\mathbb{P}[X_v=1|X_{S_v}=x_{S_v}]|\le 1/\ln(n)\]

for all $v$ and $x$ with probability $1-o(1)$. Now, let $r'=\sqrt[3]{\min(r,\ln(\ln(\ln(n))))}$ and $\delta=e^{-r'}$. One can easily check that $e^{-4r'}\tanh^{r'}(\delta)=\omega(1/\ln(\ln(n))+(bd)^{r/2})$, so by lemma \ref{highGirthCorrelation}, $S_v$ contains every vertex that is connected to $v$ by a path of length at most $r'$ in which every edge has a weight of absolute value at least $\delta$ for every $v$ with probability $1-o(1)$. If this holds, then by corollary \ref{cor2} there exists $\epsilon=e^{-\Omega(r')}$ such that
\[\left|\mathbb{P}[X_v=1|X_{S_v}=x_{S_v}]-\mathbb{P}[X_v=1|X_{S\backslash\{v\}}=x_{S\backslash\{v\}}]\right|\le \epsilon\]
for all $v$ and $x$. That means that 
\[|f_v(x)-\mathbb{P}[X_v=1|X_{S\backslash\{v\}}=x_{S\backslash\{v\}}]|\le 1/\ln(n)+\epsilon\]
for all $v$ and $x$ with probability $1-o(1)$. So, the output of $\HGL\left(n, n , S, 1/\ln(\ln(n)), \CMRF_{S,(\M,\theta)}\right)$ is within an earthmover distance of $\left(2[1/\ln(n)+2\epsilon]/\left(1-\frac{bd}{1-bd}\right)+o(1)\right)n=o(n)$ of $\CMRF_{S,(\M,\theta)}$ by theorem \ref{robustSampling} and corollary \ref{cor1}.
\end{proof}

\begin{remark}
This argument would actually show that $\HGL\left(n, n , S, 1/\ln(\ln(n)), \CMRF_{S,(\M,\theta)}\right)$ is within an earthmover distance of 
\[
ne^{-\Omega\left(\sqrt[3]{r}\right)}+ne^{-\Omega\left(\sqrt[3]{\ln(\ln(\ln(n)))}\right)}
\]
of $\CMRF_{S,(\M,\theta)}$ if we were more careful about the exact asymptotics.
\end{remark}

\section{Statistical/Computational gap for learning general CMRFs}
In the main results of this section we show that there is an exponential time algorithm that learns to sample any CMRF from a polynomial number of sample, but doing it in polynomial time is hard assuming the existence of one way functions. 

\subsection{Information-theoretic learning of general CMRFs}
Given unlimited computational resources, we could brute force the task of learning a CMRF from samples by making a list of CMRFs such that every CMRF with $n$ vertices is within total variation distance $o(1)$ of at least one of them and then checking which one best fits the observed samples. More formally, we have the following.
\begin{theorem} \label{thm:exp_time}
There exists an algorithm $A$ running in time $2^{n^{O(1)}}$ such that given any CMRF on $n$ vertices, $A$ takes a polynomial number of samples from the CMRF as input, and returns the parameters of a CMRF that is within total variation distance $O(1/n^2)$ of it with probability $1-o(1)$.
\end{theorem}

Note that the CMRF this algorithm returns will also be within an earthmover distance of $O(1/n)$ of the original CMRF with high probability. Our first step towards proving this is to show that we can use samples to narrow our list of candidate CMRFs down to one that approximates the target distribution well, which will require the following lemma, 
see. e.g.~\cite{LeCam:12} for similar statements. 

\begin{lemma}\label{SamplingLemma}
Let $k=\omega(1)$, $\epsilon,m>0$, $P$ be a probability distribution on $\{-1,1\}^m$, $P_0,...,P_k$ be probability distributions on $\{-1,1\}^m$ such that $TV(P,P_0)\le\epsilon$ and the probability of drawing $x$ from $P$ or from $P_i$ is at least $e^{-m}$ for all $x\in\{-1,1\}^m$ and all $0\le i\le k$. Now, let $X_1,...,X_T\sim P$ and choose $j$ which maximizes the probability that a series of samples drawn from $P_j$ would be $(X_1,...,X_T)$. Then $TV(P,P_j)\le \sqrt[4]{4 m^2\ln(k)/T}+\sqrt{e^m\epsilon/2}$ with probability $1-o(1)$.
\end{lemma}

\begin{proof}
First, for each $0\le j\le k$ and $x\in\{-1,1\}^m$, let $p_j(x)$ be the probability $P_j$ assigns to $x$. We know that $j$ was chosen to maximize $\prod_{t=1}^T p_i(X_t)$. Now, observe that $e^{-m}\le p_j(X_t)\le 1$ and $E[\ln(p_j(X_t))]=H(P)-D_{KL}(P||P_j)$ for all $j$ and $t$. So,
\[\mathbb{P}\left[\left|\frac{1}{T}\sum_{t=1}^T \ln(p_j(X_t))-\left(H(P)-D_{KL}(P||P_j)\right)\right|\ge \epsilon'\right]\le 2e^{-(\epsilon')^2 T/2m^2}\]
for all $0\le j\le k$ and $\epsilon'>0$ by a Chernoff bound. In particular, this means that 
\[\left|\frac{1}{T}\sum_{t=1}^T \ln(p_j(X_t))-\left(H(P)-D_{KL}(P||P_j)\right)\right|\le \sqrt{4m^2 \ln(k)/T}\]
for all $j$ with probability $1-o(1)$. We know that $\sum_{t=1}^T \ln(p_i(X_t))\ge \sum_{t=1}^T \ln(p_0(X_t))$, so with probability $1-o(1)$ it will be the case that 
\[D_{KL}(P||P_i)-D_{KL}(P||P_0)\le \sqrt{16 m^2\ln(k)/T}\]
Next, observe that $D_{KL}(P||P_0)\le e^m\cdot TV(P,P_0)\le e^m\epsilon$ because every possible value of $x$ occurs with probability at least $e^{-m}$ under both distributions. On the flip side, $D_{KL}(P||P_j)\ge 2TV^2(P,P_j)$ for all $j$. So, $TV(P,P_j)\le \sqrt[4]{4 m^2\ln(k)/T}+\sqrt{e^m\epsilon/2}$ with probability $1-o(1)$, as desired.
\end{proof}

This means that if we can find a list of $CMRFs$ of at most exponential length such that at least one of them is within a total variation distance of $o(e^{-n})$ of the desired distribution, then the brute force search will suceed at finding a $CMRF$ that is within $o(1)$ of the desired distribution. The obvious idea would be to round all edge weights and biases to the nearest multiple of $e^{-2n}$, but that could run into trouble with CMRFs in which some of the weights are superexponentially large. Instead, we will argue that the following algorithm converts every CMRF into a similar CMRF that is a member of a manageably sized list.

\begin{algorithm}

\caption{MRF List Conversion}

\begin{algorithmic}

\STATE {\bf Input:} An MRF $I$, $\epsilon>0$ and $n$ the number of vertices in $I$

\STATE {\bf Output:} An MRF that is approximately the same as $I$

\STATE Choose $x\in\{-1,1\}^n$ which maximizes $\mathbb{P}_{X\sim I}[X=x]$.

\FOR{$x'\in\{-1,1\}^n\backslash\{x\}$} 
\STATE let $r_{x'}=\epsilon \left\lfloor \frac{ \mathbb{P}_{X\sim I}[X=x']}{\epsilon \mathbb{P}_{X\sim I}[X=x]} \right\rfloor$.
\ENDFOR

\STATE Let $\overline{I}$ be the MRF such that $r_{x'}\le \frac{ \mathbb{P}_{X\sim \overline{I}}[X=x']}{ P_{X\sim \overline{I}}[X=x]}\le r_{x'}+\epsilon$ for all $x'$ with the lowest possible sum of its squared edge weights and biases.

\RETURN $\overline{I}$.
\end{algorithmic}
\end{algorithm}

We claim that this algorithm will always output an MRF that is a good approximation of its input, and that it has a manageably small number of possible outputs for fixed $\epsilon$ and $n$. A little more formally we have the following:

\begin{lemma} \label{lem:compression}
For any given $\epsilon>0$ and positive integer $n$, there are at most $2^{n^3+n}(2+1/\epsilon)^{n^2}$ possible outputs of $MRFListConversion(I,\epsilon,n)$ and for any $I$ the output of $MRFListConversion(I,\epsilon,n)$  will always be within a total variation distance of $2^n \epsilon$ of $I$.
\end{lemma}

\begin{proof}
First, observe that the inequality $r_{x'}\le \frac{ \mathbb{P}_{X\sim \overline{I}}[X=x']}{ \mathbb{P}_{X\sim \overline{I}}[X=x]}\le r_{x'}+\epsilon$ is equivalent to a pair of linear inequalities on the parameters of $\overline{I}$ for all $x,x', r_{x'}$. Also, $I$ will satisfy it. As such, for any given value of $x$ and the $r_{x'}$, there will be a unique set of parameters with minimum sum of squares that satisifes all of these constraints. Furthermore, there will be some set of equations of the form $\frac{ \mathbb{P}_{X\sim \overline{I}}[X=x']}{ \mathbb{P}_{X\sim \overline{I}}[X=x]}=r_{x'}$ or  $\frac{ \mathbb{P}_{X\sim \overline{I}}[X=x']}{ \mathbb{P}_{X\sim \overline{I}}[X=x]}=r_{x'}+\epsilon$ such that the parameters of $\overline{I}$ will be the least square solution to this system of equations.

An MRF on $n$ variables has $n+n(n-1)/2$ parameters, so any such system of equations is equivalent to a system of $n^2$ or fewer equations. $r_n$ will always be a multiple of $\epsilon$ with $0\le r_{x'}\le 1$ and for a fixed value of $x$, there are $(2^n-1)$ possible values of $x'$, so there are $(2^n-1)(2+1/\epsilon)$ possibilities for each equation. That means that there are $2^n$ possible values of $x$ and at most $[2^n(2+1/\epsilon)]^{n^2}$ possible systems of equations for a given value of $x$. That in turn means that there are at most $2^{n^3+n}(2+1/\epsilon)^{n^2}$ possible outputs of $MRFListConversion(I,\epsilon,n)$ as desired.

Now, let $\overline{I}=MRFListConversion(I,\epsilon,n)$ and observe that $\left| \frac{ \mathbb{P}_{X\sim \overline{I}}[X=x']}{ \mathbb{P}_{X\sim \overline{I}}[X=x]}- \frac{ \mathbb{P}_{X\sim I}[X=x']}{ \mathbb{P}_{X\sim I}[X=x]}\right|\le\epsilon$ for all $x'$. If $\mathbb{P}_{X\sim \overline{I}}[X=x]\ge \mathbb{P}_{X\sim I}[X=x]$ then
\begin{align*}
TV(I,\overline{I})&=\sum_{x'\in\{-1,1\}^n} \max\left(\mathbb{P}_{X\sim I}[X=x']-\mathbb{P}_{X\sim \overline{I}}[X=x'],0\right)\\
&\le \sum_{x'\in\{-1,1\}^n} \max\left(\mathbb{P}_{X\sim I}[X=x']-\mathbb{P}_{X\sim \overline{I}}[X=x'],0\right)/\mathbb{P}_{X\sim I}[X=x]\\
&\le \sum_{x'\in\{-1,1\}^n} \max\left(\frac{\mathbb{P}_{X\sim I}[X=x']}{\mathbb{P}_{X\sim I}[X=x]}-\frac{\mathbb{P}_{X\sim \overline{I}}[X=x']}{\mathbb{P}_{X\sim\overline{ I}}[X=x]},0\right)\\
&\le 2^n\epsilon
\end{align*}
Similarly, if  $\mathbb{P}_{X\sim \overline{I}}[X=x]< \mathbb{P}_{X\sim I}[X=x]$ then 
\[TV(I,\overline{I})=\sum_{x'\in\{-1,1\}^n} \max\left(\mathbb{P}_{X\sim \overline{I}}[X=x']-\mathbb{P}_{X\sim I}[X=x'],0\right)/\mathbb{P}_{X\sim \overline{I}}[X=x]\le 2^n\epsilon\]
So, either way the total variation distance between the $I$ and $\overline{I}$ will be at most $2^n\epsilon$ as desired.
\end{proof}

Combining Lemma~\ref{SamplingLemma} with Lemma~\ref{lem:compression} allows us to prove the theorem:

\begin{proof}
Given a value $n$ and the ability to sample from an unknown CMRF on $n$ vertices, $I$ , $A$ will do the following. First, it will find all possible outputs of $MRFListConversion(I',8^{-n},n)$ and then make a list $I_1,...,I_r$ of all possible censorings of these MRFs with the same number of visible vertices as $I$, $m$. This list will have at most $2^{n^3+2n}(8^n+2)^{n^2}\le 2^{4n^3+2n^2+2n}$ elements, and at least one of them will be within a total variation distance of $4^{-n}$ of $I$. 

Now, let $I^\star$ be the probability distribution attained by returning a random element of $\{-1,1\}^m$ with probability $(2/e)^m$ and taking a random sample from $I$ otherwise. Similarly, for each $1\le j\le r$, let $I^\star_j$ be the probability distribution attained by returning a random element of $\{-1,1\}^m$ with probability $(2/e)^m$ and taking a random sample from $I_j$ otherwise. Next, let $T=4m^2(4n^3+2n^2+2n)n^8$ and draw $X_1,...,X_T\sim I^\star$. Then, return $I_j$ where $j$ is chosen to maximize the probability that a series of samples drawn from $I^\star_j$ would be $(X_1,...,X_T)$.

There must exist $k$ such that $TV(I^\star,I^\star_k)\le 4^{-n}$, so by lemma \ref{SamplingLemma} the total variation distance between $I^\star$ and $I^\star_j$ will be at most $\sqrt[4]{1/n^8}+(e/4)^{n/2}$ with probability $1-o(1)$. That in turn means that $TV(I,I_k)=O(1/n^2)$ with probability $1-o(1)$. So, this algorithm suceeds in returning a CMRF that is within total variation distance $O(1/n^2)$ of the target distribution with probability $1-o(1)$, as desired.
\end{proof}

\subsection{The computational hardness of learning a general CMRF}
The algorithm from the previous subsection succeeds in learning an arbitrary CMRF with vanishing error given a polynomial number of samples; however, it has an exponentially large run time. Is there such an algorithm that runs in polynomial time? 
We would like to find an efficient algorithm that learns any CMRF with earthmover distortion $o(n)$. However, this probably does not exist. In order to demonstrate that, we will show that we can construct a CMRF that assigns values to its visible vertices by means of an arbitrary efficient randomized algorithm. Then we prove that this implies it could set the visible vertices pseudorandomly in which case it would be computationally intractible to learn their probability distribution. More formally we prove that: 

\begin{theorem} \label{thm:1way} 
Let $A$ be an algorithm that attempts to learn a CMRF from samples, runs in time polynomial in the total number of vertices in the CMRF, and outputs a new value. If one way functions exist, then there exists a family $I_n$ of CMRFs with $n$ visible vertices such that $W_{I_n}(A,I_n) = (1/2-o(1))n$. 
\end{theorem}


For the theorem we will use the following definition of one-way functions., see e.g.~\cite{GoGoMi:86}. 
\begin{definition}
Given a sequence of domains $D_k\subseteq \{0,1\}^k$ and of functions $f_k:D_k\rightarrow D_k$, $f$ is a one-way function if the following criteria hold:
\begin{enumerate}
\item There exists an algorithm that runs in $poly(k)$ time and computes $f_k(x)$ for all $k$ and $x$.

\item Given any polynomial time randomized algorithm $A$, any $1\le i\le k^3$, and a random $x$ in the image of the $i$ times composition of $f_k$, the probability that $f(A(x))=x$ is bounded away from $1$ for all sufficiently large $k$.

\item $\cup D_k$ is samplable.
\end{enumerate}
\end{definition}

In the rest of the section we prove Theorem~\ref{thm:1way}. One important ingredient in the proof is the fact that we can encode an arbitrary efficient computation, i.e.,  an arbitrary circuit in an MRF. This is not a new idea, see e.g.~\cite{BoMoVa:08}, where a similar results was proven for the hard-core model. For completeness we include the proof of the following in the appendix: 

\begin{lemma} \label{lem:circuit_encoding} 
Let $A_n$ be an efficient randomized algorithm that samples from some probability distribution on $\{-1,1\}^n$. Then there exists a series of CMRFs, $M_n$, such that $M_n$ has $n$ visible vertices, a total number of vertices polynomial in $n$, and a probability distribution that is within a total variation distance of $O(e^{-n})$ from the probability distribution of the output of $A_n$.
\end{lemma}

We now prove Theorem~\ref{thm:1way}. 

\begin{proof}
\cite{GoGoMi:86} shows that if one way functions exist then there exists a pseudorandom function family $f_n: \{0,1\}^n\times\{0,1\}^n\rightarrow \{0,1\}^n$ such that $f$ is efficiently computable and there is no efficient algorithm that can distinguish $f(s,\cdot)$ from a true random function for an unknown random $s\in \{0,1\}^n$. 
In other words, the probability that an efficient algorithm can distinguish between the output of $f(s,\cdot)$ and a truly random function, is asymptotically smaller $n^{-C}$ for any $C$. This implies in particular, that for any $m \leq 2^n$, an efficient algorithm cannot distinguish between the distribution of
$( f(s,U_1), f(s,U_2),\ldots,f(s,U_t))$ and $(y_{U_1},\ldots,y_{U_t})$ where $U_i$ are i.i.d. uniform in $0,\ldots,m-1$.  

Let $A = A_n$ be an algorithm (or a sequence of circuits) which uses $O(n^c)$ samples and attempts to learn to sample a CMRF with $n$ observed nodes. 
By lemma~\ref{lem:circuit_encoding}, for any given $n>0$, $s\in \{0,1\}^n$, and $0< m\le 2^n$ there exists a CMRF $I_{(n,s,m)}$ with size polynomial in $n$, with $n$ visible vertices that are almost always assigned to a value corresponding to $f(s,x)$ for some random $0\le x< m$ (this is true since we can efficiently sample 
$0 \leq x \leq m$ and then efficiently compute $f(s,x)$). 

Now, consider randomly selecting $s\in\{0,1\}^n$ and using $A$ to attempt to learn $I_{(n,s,n^{c+1})}$. Also, randomly select $x_1,...,x_{n^{c+1}}\in\{0,1\}^n$ and let $I'$ be the probability distribution that selects one of the $x_i$ uniformly at random. The pseudorandomness of $f$ implies that no efficient algorithm can distinguish between $I_{(n,s,n^{c+1})}$ and $I'$ with nonvanishing advantage. 
If we ran $A$ on $I'$ then its output would have to have a hamming distance of at least $n/2-O(\sqrt{n}\log(n))$ from all the $x_i$ it had not seen with probability $1-o(1)$ simply because any element of $\{0,1\}^n$ is at least that far from the closest of $n^{c+1}$ random values with high probability. Furthermore, given $O(n^{c+1}\log(n))$ additional random samples from $I'$ or $I_{(n,s,n^{c+1})}$ one can determine the distance between the output of $A$ and the closest value of $I'$ or $I_{(n,s,n^{c+1})}$ that was not included in the set of samples $A$ received. So, the fact that one can not efficiently distinguish $I_{(n,s,n^{c+1})}$ from $I'$ implies that when $A$ attempts to learn $I_{(n,s,n^{c+1})}$ its output is at least $n/2-O(\sqrt{n}\log(n))$ away from all values of  $I_{(n,s,n^{c+1})}$ that it has not seen with high probability. The values that it has seen account for $o(1)$ of the probability distribution, so this implies that for a fixed set of samples $A$ could have received, the probability distribution of its output is an earthmover distance of $n/2-o(n)$ away from $I_{(n,s,n^{c+1})}$ with high probability.
\end{proof}

\begin{remark} 
One potential criticism of the result above is that the overwhelming majority of the vertices are censored. It is easy to modify the reduction by giving each visible vertex a large set of additional visible vertices that are equal to it with high probability. In this case, most of the vertices would be visible, and it would still be intractable to find a probability distribution within a nontrivial earthmover distance of it. 
\end{remark} 

\begin{remark}
This theorem proves that no efficient learning algorithm can find a probability distribution within a nontrivial earthmover distance of an arbitrary CMRF. One could try to find an algorithm that learns a CMRF based on some other criterion, but there are not a lot of obvious options. Slight variations of the argument in the theorem also show that it is impossible for any efficient algorithm to learn to estimate the probability distribution of a vertex conditioned on the values of any large subset of other vertices.
 \end{remark}

\bibliographystyle{plain}
\bibliography{all,my}

\begin{thebibliography}{10}

\bibitem{AbKoNg:06}
P.~Abbeel, D.~Koller, and A.~Y. Ng.
\newblock Learning factor graphs in polynomial time and sampling complexity.
\newblock {\em Journal of Machine Learning Research}, 7:1743--1788, 2006.

\bibitem{AnandkumarValluvan:13}
Animashree Anandkumar, Ragupathyraj Valluvan, et~al.
\newblock Learning loopy graphical models with latent variables: Efficient
  methods and guarantees.
\newblock {\em The Annals of Statistics}, 41(2):401--435, 2013.

\bibitem{BoMoVa:08}
A.~Bogdanov, E.~Mossel, and S.~Vadhan.
\newblock The complexity of distinguishing markov random fields.
\newblock In {\em 11th International Workshop, APPROX 2008, and 12th
  International Workshop, RANDOM 2008, LNCS 5171}, pages 331--342. Springer,
  2008.

\bibitem{BrMoSl:08}
G.~Bresler, E.~Mossel, and A.~Sly.
\newblock Reconstruction of markov random fields from samples: Some easy
  observations and algorithms.
\newblock In {\em 11th International Workshop, APPROX 2008, and 12th
  International Workshop, RANDOM 2008, LNCS 5171}, pages 343--356. Springer,
  2008.

\bibitem{Bresler:15}
Guy Bresler.
\newblock Efficiently learning ising models on arbitrary graphs.
\newblock In {\em Proceedings of the Forty-Seventh Annual ACM on Symposium on
  Theory of Computing}, pages 771--782. ACM, 2015.

\bibitem{BrKoMo:19}
Guy Bresler, Frederic Koehler, and Ankur Moitra.
\newblock Learning restricted boltzmann machines via influence maximization.
\newblock In {\em Proceedings of the 51st Annual ACM SIGACT Symposium on Theory
  of Computing}, pages 828--839, 2019.

\bibitem{Chang:96}
J.~Chang.
\newblock Full reconstruction of markov models on evolutionary trees:
  identifiability and consistency.
\newblock {\em Math. Biosci.}, 137(51--73), 1996.

\bibitem{ChowLiu:68}
C~Chow and Cong Liu.
\newblock Approximating discrete probability distributions with dependence
  trees.
\newblock {\em IEEE transactions on Information Theory}, 14(3):462--467, 1968.

\bibitem{DanielyShalevShwartz:16}
Amit Daniely and Shai Shalev-Shwartz.
\newblock Complexity theoretic limitations on learning dnf's.
\newblock In {\em COLT}, pages 815--830, 2016.

\bibitem{DanielyVardi:20}
Amit Daniely and Gal Vardi.
\newblock Hardness of learning neural networks with natural weights.
\newblock {\em arXiv preprint arXiv:2006.03177}, 2020.

\bibitem{DobrushinShlosman:85}
R.~L. Dobrushin and S.~B. Shlosman.
\newblock Constructive criterion for uniqueness of a {G}ibbs field.
\newblock In J.~Fritz, A.~Jaffe, and D.~Szasz, editors, {\em Statistical
  Mechanics and dynamical systems}, volume~10, pages 347--370. 1985.

\bibitem{Dobrushin:70}
Roland~L Dobrushin.
\newblock Prescribing a system of random variables by conditional
  distributions.
\newblock {\em Theory of Probability \& Its Applications}, 15(3):458--486,
  1970.

\bibitem{ErStSzWa:99a}
P.~L. Erd\"{o}s, M.~A. Steel;, L.~A. Sz\'{e}kely, and T.~A. Warnow.
\newblock A few logs suffice to build (almost) all trees (part 1).
\newblock {\em Random Structures Algorithms}, 14(2):153--184, 1999.

\bibitem{GoGoMi:86}
O~GOLDREICH, S~GOLDWASSER, and S~MICALI.
\newblock How to construct random functions.
\newblock {\em Journal of the Association for Computing Machinery},
  33(4):792--807, 1986.

\bibitem{GrReVe:11}
Elena Grigorescu, Lev Reyzin, and Santosh Vempala.
\newblock On noise-tolerant learning of sparse parities and related problems.
\newblock In {\em International Conference on Algorithmic Learning Theory},
  pages 413--424. Springer, 2011.

\bibitem{hinton2009deep}
Geoffrey~E Hinton.
\newblock Deep belief networks.
\newblock {\em Scholarpedia}, 4(5):5947, 2009.

\bibitem{KlivansMeka:17}
Adam Klivans and Raghu Meka.
\newblock Learning graphical models using multiplicative weights.
\newblock In {\em 2017 IEEE 58th Annual Symposium on Foundations of Computer
  Science (FOCS)}, pages 343--354. IEEE, 2017.

\bibitem{KlivansSherstov:09}
Adam~R. Klivans and Alexander~A. Sherstov.
\newblock Cryptographic hardness for learning intersections of halfspaces.
\newblock {\em J. Comput. Syst. Sci.}, 75(1):2--12, 2009.

\bibitem{Kunsch:82}
H~K{\"u}nsch.
\newblock Decay of correlations under dobrushin's uniqueness condition and its
  applications.
\newblock {\em Communications in Mathematical Physics}, 84(2):207--222, 1982.

\bibitem{LeCam:12}
Lucien Le~Cam.
\newblock {\em Asymptotic methods in statistical decision theory}.
\newblock Springer Science \& Business Media, 2012.

\bibitem{Martinelli:99}
F.~Martinelli.
\newblock Lectures on {G}lauber dynamics for discrete spin models.
\newblock In {\em Lectures on probability theory and statistics (Saint-Flour,
  1997)}, volume 1717 of {\em Lecture Notes in Math.}, pages 93--191. Springer,
  Berlin, 1999.

\bibitem{Mossel:07}
E.~Mossel.
\newblock Distorted metrics on trees and phylogenetic forests.
\newblock {\em IEEE Computational Biology and Bioinformatics}, 4:108--116,
  2007.

\bibitem{MosselRoch:05}
E.~Mossel and S.~Roch.
\newblock Learning nonsingular phylogenies and hidden markov models.
\newblock In {\em Proceedings of the thirty-seventh annual ACM symposium on
  Theory of computing, Baltimore (STOC05), MD, USA}, pages 366--376, 2005.

\bibitem{Sly:10}
A.~Sly.
\newblock Computational transition at the uniqueness threshold.
\newblock In {\em Foundations of Computer Science (FOCS)}, pages 287--296,
  2010.

\bibitem{SlySun:12}
Allan Sly and Nike Sun.
\newblock The computational hardness of counting in two-spin models on
  d-regular graphs.
\newblock In {\em 2012 IEEE 53rd Annual Symposium on Foundations of Computer
  Science}, pages 361--369. IEEE, 2012.

\bibitem{WaRaLa:06}
M.~J. Wainwright, P.~Ravikumar, and J.~D. Lafferty.
\newblock High dimensional graphical model selection using $\ell_1$-regularized
  logistic regression.
\newblock In {\em Proceedings of the NIPS}, 2006.

\bibitem{Weitz:04}
Dror Weitz and Alistair Sinclair.
\newblock {\em Mixing in time and space for discrete spin systems}.
\newblock University of California, Berkeley, 2004.

\end{thebibliography}

\appendix

\section{Glauber Dynamics and Correlation Decay}

\begin{theorem}\label{correlationDecay}
Let $b>0$ and $d$ and $n$ be positive integers such that $bd<1$. Next, let $\theta\in\mathbb{R}^n$ and $\M$ be an $n\times n$ symmetric matrix such that $\M_{i,i}=0$ for all $i$, $|\M_{i,j}|\le b$ for all $i$ and $j$, and for each $i$ there are at most $d$ values of $j$ for which $\M_{i,j}\ne 0$. Next, let $S\subseteq [n]$, $x\in\{-1,1\}^n$, $v, u\not\in S$. Also, let $\N \in\mathbb{R}^{n\times n}$ such that $\N_{i,j}=|\M_{i,j}|$ if $i,j\not\in S$ and $0$ otherwise. Then if $X\sim \MRF_{(\M,\theta)}$, 
\[|\mathbb{P}\left[x_v=1|X_{S}=x_{S}, X_{u}=1\right]-\mathbb{P}\left[x_v=1|X_{S}=x_{S}, X_{u}=-1\right]|\le \sum_{k=0}^\infty {\N}^k_{v,u}\]
where we consider ${\N}^0$ to be the identity matrix.
\end{theorem} 

\begin{proof}
In order to prove this, we are going to use $\MRFMCMC$ to draw samples from the appropriate MRF conditioned on both partial assignments and compare them. First, define $x^{(0)}$ and $y^{(0)}$ such that $x^{(0)}_{u}=1$, $y^{(0)}_{u}=-1$, and $x^{(0)}_i=y^{(0)}_i=x_i$ for all $i\ne u$. Next, let $S'=[n]\backslash (S \cup \{u\})$, and randomly select $v^{(t)}\in S'$ and $p^{(t)}\in [0,1]$ for each $t\ge 0$. Now, for each $t> 0$, let $x^{(t)}$ be the value returned by $\MRFMCMC(\M,\theta, x^{(0)}, t, S')$ if the algorithm selects $v^{(t')}$ and $p^{(t')}$ in step $t'$ for each $0\le t'<t$. Likewise, for each $t>0$, let $y^{(t)}$ be the value returned by $\MRFMCMC(\M,\theta, y^{(0)}, t, S')$ if the algorithm selects $v^{(t')}$ and $p^{(t')}$ in step $t'$ for each $0\le t'<t$. This leaves the probability distribution of the algorithm's randomness unchanged, so for any $w$,
\[\lim_{t\to\infty} \mathbb{P}\left[x^{(t)}_{w}=1\right]=\mathbb{P}_{X\sim MRF_{\M,\theta}}[X_{w}=1|X_S=x_S,X_{u}=1]\]
and
\[\lim_{t\to\infty} \mathbb{P}\left[y^{(t)}_{w}=1\right]=\mathbb{P}_{X\sim MRF_{\M,\theta}}[X_{w}=1|X_S=x_S,X_{u}=-1]\]

Now, for each $t$ and $w$, let $P^{(t)}_{w}=\mathbb{P}\left[x^{(t)}_{w}\ne y^{ (t)}_{w}\right]$, and observe that for any $t>0$ and $w\in S'$, 

\[P^{(t)}_{w} \le \frac{|S'|-1}{|S'|} P^{(t-1)}_{w}+\frac{1}{|S'|}\sum_{w'\in [n]} |\M_{w,w'}|\cdot P^{(t-1)}_{w'}\]

So,
\begin{align*}
\limsup_{t\to\infty} P^{(t)}_{w}&\le\limsup_{t\to\infty} \sum_{w'\in [n]} |\M_{w,w'}|\cdot P^{(t)}_{w'}\\
&=\limsup_{t\to\infty} \left[|\M_{w,u}|\cdot P^{(t)}_{u}+\sum_{w'\in S} |\M_{w,w'}|\cdot P^{(t)}_{w'}+\sum_{w'\in S'} |\M_{w,w'}|\cdot P^{(t)}_{w'}\right]\\
&=\limsup_{t\to\infty} \left[|\M_{w,u}|\cdot 1+\sum_{w'\in S} |\M_{w,w'}|\cdot 0+\sum_{w'\in S'} |\M_{w,w'}|\cdot P^{(t)}_{w'}\right]\\
&=|\M_{w,u}|+\limsup_{t\to\infty} \sum_{w'\in S'} |\M_{w,w'}|\cdot P^{(t)}_{w'}\\
&=\N_{w,u}+\limsup_{t\to\infty} \sum_{w'\in S'} \N_{w,w'}\cdot P^{(t)}_{w'}\\
\end{align*}
Applying this repeatedly gives us that
\[\limsup_{t\to\infty} P^{(t)}_{w}\le \sum_{k=1}^\infty {\N}^k_{w,u}\]
In particular, if $v\ne u$ then
\begin{align*}
&|\mathbb{P}\left[x_v=1|X_{S}=x_{S}, X_{u}=1\right]-\mathbb{P}\left[x_v=1|X_{S}=x_{S}, X_{u}=-1\right]|\\
&=\left| \lim_{t\to\infty} \left(\mathbb{P}\left[x^{(t)}_{v}=1\right]-\mathbb{P}\left[y^{(t)}_{v}=1\right]\right)\right|\\
&\le \limsup_{t\to\infty}P^{(t)}_{v}\\
&\le  \sum_{k=1}^\infty {\N}^k_{v,u}
\end{align*}
Also, if $v=u$ then 
\[|\mathbb{P}\left[x_v=1|X_{S}=x_{S}, X_{u}=1\right]-\mathbb{P}\left[x_v=1|X_{S}=x_{S}, X_{u}=-1\right]|=1\le \sum_{k=0}^\infty {\N}^k_{v,v} .\]
\end{proof}

\begin{corollary}
Let $b>0$ and $d$ and $n$ be positive integers such that $bd<1$. Next, let $\theta\in\mathbb{R}^n$ and $\M$ be an $n\times n$ symmetric matrix such that $\M_{i,i}=0$ for all $i$, $|\M_{i,j}|\le b$ for all $i$ and $j$, and for each $i$ there are at most $d$ values of $j$ for which $\M_{i,j}\ne 0$. Next, let $S\subseteq [n]$, $x\in\{-1,1\}^n$, $v, u\not\in S$. Then
\[\sum_{u\in S: u\ne v} \left|\mathbb{P}[X_{u}=1|X_{S\backslash \{v,u\}}=x_{S\backslash \{v,u\}},X_v=1]-\mathbb{P}[X_{u}=1|X_{S\backslash \{v,u\}}=x_{S\backslash \{v,u\}},X_v=-1]\right|\le bd/(1-bd)\]
\end{corollary}

\begin{proof}
First of all, let $\N \in\mathbb{R}^{n\times n}$ such that ${\N}_{i,j}=|\M_{i,j}|$ for all $i$ and $j$. Also, for each $v,u\in S$, let $\M^{(v,u)}\in\mathbb{R}^{n\times n}$ such that $\M^{(v,u)}_{i,j}=|A_{i,j}|$ if $i,j\not\in S\backslash\{v,u\}$ and $0$ otherwise.
\begin{align*}
&\sum_{u\in S: u\ne v} \left|\mathbb{P}[X_{u}=1|X_{S\backslash \{v,u\}}=x_{S\backslash \{v,u\}},X_v=1]-\mathbb{P}[X_{u}=1|X_{S\backslash \{v,u\}}=x_{S\backslash \{v,u\}},X_v=-1]\right|\\
& \qquad \qquad \qquad \qquad \le \sum_{u\in S:u\ne v}\sum_{k=0}^\infty (\M^{(v,u)})^k_{v,u}  \le \sum_{u\in S:u\ne v}\sum_{k=0}^\infty {\N}^k_{v,u}\\
& \qquad \qquad \qquad \qquad = \sum_{u\in S:u\ne v}\sum_{k=1}^\infty {\N}^k_{v,u} = \sum_{k=1}^\infty \sum_{u\in S:u\ne v} {\N}^k_{v,u} \le \sum_{k=1}^\infty (bd)^k  \le bd/(1-bd) .
\end{align*}
\end{proof}


\begin{corollary}
Let $b,r,\delta\ge 0$ and $d$ and $n$ be positive integers such that $bd<1$. Next, let $\theta\in\mathbb{R}^n$ and $A$ be an $n\times n$ symmetric matrix such that $A_{i,i}=0$ for all $i$, $|A_{i,j}|\le b$ for all $i$ and $j$, and for each $i$ there are at most $d$ values of $j$ for which $A_{i,j}\ne 0$. Next, let $S\subseteq [n]$, $v\in S$, $x\in \{-1,1\}^n$, and $G$ be the graph with adjacency matrix $A$. Now, let $S'$ be a subset of $S$ which contains as a subset the set of all vertices that are connected to $v$ in $G$ by a path of length less than $r$ in which every edge has a weight with an absolute value of at least $\delta$, except for $v$ itself. Also, let $X\sim \CMRF_{S(\M,\theta)}$. Then
\[\left|\mathbb{P}[X_v=1|X_{S\backslash \{v\}}=x_{S\backslash \{v\}}]-\mathbb{P}[X_v=1|X_{S'}=x_{S'}]\right|\le [d\delta+(bd)^r]/(1-bd)\]
\end{corollary}

\begin{proof}
Let $\N \in\mathbb{R}^{n\times n}$ such that $\N_{i,j} =|\M_{i,j}|$ for all $i$ and $j$. Next, for each $u\in S$, let $\M^{(u)}\in\mathbb{R}^{n\times n}$ such that $\M^{(u)}_{i,j}$ is $|\M_{i,j}|$ if $i,j\not\in S\backslash\{v,u\}$ and $0$ otherwise. Now, let $x'\in\{-1,1\}^n$ such that $x'_{u}=x_{u}$ for all $u\in S'$. Next, for each $0\le u\le n$, let $x^{(u)}\in\mathbb{R}^n$ such that 
\[x^{(u)}_{w}=
\begin{cases}
x_{w}  &\text{ if } w> u\\
x'_{w} &\text{ if } w\le u
\end{cases}
\]
Now, observe that writing $U = S \backslash (S' \cup \{v\})$, 
\begin{align*}
&\left|\mathbb{P}\left[X_v=1\middle|X_{S\backslash \{v\}}=x_{S\backslash \{v\}}\right]-\mathbb{P}\left[X_v=1\middle|X_{S\backslash \{v\}}=x'_{S\backslash \{v\}}\right]\right|\\
&=\left|\mathbb{P}\left[X_v=1\middle|X_{S\backslash \{v\}}=x^{(0)}_{S\backslash \{v\}}\right]-\mathbb{P}\left[X_v=1\middle|X_{S\backslash \{v\}}=x^{(n)}_{S\backslash \{v\}}\right]\right|\\
&\le \sum_{w=1}^n \left|\mathbb{P}\left[X_v=1\middle|X_{S\backslash \{v\}}=x^{(w-1)}_{S\backslash \{v\}}\right]-\mathbb{P}\left[X_v=1\middle|X_{S\backslash \{v\}}=x^{(w)}_{S\backslash \{v\}}\right]\right|\\
&= \sum_{w\in U} \left|\mathbb{P}\left[X_v=1\middle|X_{S\backslash \{v\}}=x^{(w-1)}_{S\backslash \{v\}}\right]-\mathbb{P}\left[X_v=1\middle|X_{S\backslash \{v\}}=x^{(w)}_{S\backslash \{v\}}\right]\right|
\end{align*}
By Theorem \ref{correlationDecay}, this is at most
\begin{align*}
\sum_{w\in U}\sum_{k=0}^{\infty} \left(\M^{(w)}\right)^k_{v,w} &\le \sum_{w\in U}\sum_{k=0}^{\infty} \left(\N \right)^k_{v,w}\\
&\le \sum_{k=0}^{\infty}\sum_{w\in U} \left(\N \right)^k_{v,w}\\
&= \sum_{k=1}^{r-1}\sum_{w\in U}\left(\N \right)^k_{v,w}+
\sum_{k=r}^{\infty}\sum_{w\in U} \left(\N \right)^k_{v,w}\\
&\le \sum_{k=1}^{r-1}d^k \delta b^{k-1}+\sum_{k=r}^{\infty} (bd)^k\\
&\le \frac{d\delta}{1-bd}+\frac{(bd)^r}{1-bd} =\frac{d\delta+(bd)^r}{1-bd}
\end{align*}

Now, observe that $\mathbb{P}[X_v=1|X_{S'}=x_{S'}]$ is a weighted average of expressions of the form $\mathbb{P}[X_v=1|X_{S'}=x_{S'}, X_U=x'']$. By the previous argument, every such probability must be within $[d\delta+(bd)^r]/(1-bd)$ of $\mathbb{P}[X_v=1|X_{S\backslash \{v\}}=x_{S\backslash \{v\}}]$, so $\mathbb{P}[X_v=1|X_{S'}=x_{S'}]$ is as well.
\end{proof}

\begin{theorem}\label{robustSampling}
Let $0<\epsilon,\beta<1$ be constants. Next, let $n$ be a positive integer and $X\in \{-1,1\}^n$ be a random variable such that for all $v\in [n]$ and $x\in\{-1,1\}^n$,
\[\sum_{u\in [n]\backslash\{v\}} \left|\mathbb{P}[X_{u}=1|X_{-\{v,u\}}=x_{-\{v,u\}},X_v=1]-\mathbb{P}[X_{u}=1|X_{-\{v,u\}}=x_{-\{v,u\}},X_v=-1]\right|\le \beta\]
Next, for each $v\in [n]$, let $f:\{-1,1\}^{n-1}\rightarrow [0,1]$ be a function such that
\[\sum_{v=1}^n \left|f_v(x_{-v})-\mathbb{P}[X_v=1|X_{-v}=x_{-v}]\right|\le \epsilon n\]
for all $x\in \{-1,1\}^{n}$. Then the probability distribution of the output of $\ApproximateMRFMCMC(f, x^{(0)},n\ln(n))$ is within an earthmover distance of $(2\epsilon/(1-\beta)+o(1))n$ of the probability distribution of $X$.
\end{theorem}

\begin{proof}
First, let $x^{(0)}=y^{(0)}=\{1\}^n$ and $z^{ (0)}$ be a random sample from the probability distribution of $X$. Also, let $f^{\star}_v(x)=\mathbb{P}[X_v=1|X_{-v}=x]$ for all $v\in [n]$ and $x\in \{-1,1\}^{n-1}$. The probability distribution of the output of  $\ApproximateMRFMCMC(f^\star, z^{(0)},n\ln(n))$ is the same as the probability distribution of $X$ because the probability distribution of $z^{ (0)}$ is the same as the probability distribution of $X$, and picking a random vertex and drawing a new value for it from its probability distribution given the values of the other vertices has no effect on the probability distribution of $x$. We plan to establish a coupling to show that the probability distributions of the outputs of $\ApproximateMRFMCMC(f, x^{(0)},n\ln(n))$, \\$\ApproximateMRFMCMC(f^\star, y^{(0)},n\ln(n))$, and \\$ApproximateMRFMCMC(f^\star, z^{(0)},n\ln(n))$ are nearly the same, thus showing that they are all aproximately equal to the probability distribution of $X$.

In order to do that, we start by randomly selecting $v^{(t)}\in [n]$ and $p^{(t)}\in [0,1]$ for each $t\ge 0$. Next, for each $t>0$, let $x^{(t)}$, $y^{(t)}$, and $z^{(t)}$ be the outputs of $\ApproximateMRFMCMC(f, x^{(0)},t)$, $\ApproximateMRFMCMC(f^\star, y^{ (0)},t)$, and $\ApproximateMRFMCMC(f^\star, z^{(0)},t)$ respectively if the algorithms select $v^{(t')}$ and $p^{(t')}$ in step $t'$ for each $0\le t'<t$. Now, observe that for each $t>0$,
\begin{align*}
&E[||z^{(t)}-y^{ (t)}||_1]\\
&\le \frac{n-1}{n}E[||z^{(t-1)}-y^{ (t-1)}||_1]+\frac{1}{n}\sum_{v=1}^n 2 E\left[\left|\mathbb{P}\left[X_v=1|X_{-v}=z^{(t-1)}_{-v}\right]-\mathbb{P}\left[X_v=1|X_{-v}=y^{ (t-1)}_{-v}\right]\right|\right]\\
&\le \frac{n-1}{n}E[||z^{(t-1)}-y^{ (t-1)}||_1]+\frac{\beta}{n}E[||z^{(t-1)}-y^{ (t-1)}||_1]\\
\end{align*}
So, $E[||z^{ (t)}-y^{ (t)}||_1]\le e^{-(1-\beta)t/n}\cdot E[||z^{ (0)}-y^{ (0)}||_1]\le 2n e^{-(1-\beta)t/n}$ for all $t$.
Similarly, for all $t>0$,

\begin{align*}
&E[||x^{(t)}-y^{(t)}||_1]\\
&\le \frac{n-1}{n}E[||x^{(t-1)}-y^{ (t-1)}||_1]+\frac{1}{n}\sum_{v=1}^n 2 E\left[\left|f_v\left (x^{(t-1)}_{-v}\right)-\mathbb{P}\left[X_v=1|X_{-v}=y^{ (t-1)}_{-v}\right]\right|\right]\\
&\le \frac{n+\beta-1}{n}E[||x^{(t-1)}-y^{ (t-1)}||_1]+ \frac{1}{n}\sum_{v=1}^n 2 E\left[\left|f_v\left (x^{(t-1)}_{-v}\right)-\mathbb{P}\left[X_v=1|X_{-v}=x^{(t-1)}_{-v}\right]\right|\right]\\
&\le \frac{n+\beta-1}{n}E[||x^{(t-1)}-y^{ (t-1)}||_1]+ 2\epsilon\\
\end{align*}

We already know that $||x^{(0)}-y^{ (0)}||_1=0$, so by induction on $t$, it must be the case that $E[||x^{(t)}-y^{ (t)}||_1]\le 2\epsilon n/(1-\beta)$ for all $t$. In particular, if we set $t=\lceil n\ln(n)\rceil$ then we have that
\begin{align*}
&E[||x^{(t)}-z^{ (t)}||_1]\\
&\le E[||x^{(t)}-y^{ (t)}||_1]+E[||y^{(t)}-z^{ (t)}||_1]\\
&\le 2\epsilon n/(1-\beta)+ 2n^{\beta}\\
&=(2\epsilon/(1-\beta)+o(1))n
\end{align*}
as desired. 
\end{proof}

In particular, both standard high temperature MRFs and high temperature CMRFs satisfy the conditions for this to apply. So, given $X$ drawn from one of these models, if we can learn to estimate the probability distributions of most of the vertices given the values of the other vertices, then we can get a reasonable approximation of the overall probability distribution of $X$.

\section{Encoding Circuits with Ising Models}
For completeness we recall the statement of Lemma~\ref{lem:circuit_encoding}.
\begin{lemma} 
Let $A_n$ be an efficient randomized algorithm that samples from some probability distribution on $\{-1,1\}^n$. Then there exists a series of CMRFs, $M_n$, such that $M_n$ has $n$ visible vertices, a total number of vertices polynomial in $n$, and a probability distribution that is within a total variation distance of $O(e^{-n})$ from the probability distribution of the output of $A_n$.
\end{lemma}

 Our first step is to construct a NAND gadget, such as the following.

\begin{definition}
For any $\delta>0$, let $J_\delta$ be the weighted graph defined as follows. $J_\delta$ has $3$ vertices, $v_1$, $v_2$, and $v_3$. There is an edge of weight $-\delta$ between $v_1$ and $v_2$ and edges of weight $-2\delta$ between the other pairs of edges. When using this structure in an MRF, we will also give $v_1$ and $v_2$ biases of $\delta$ and $v_3$ a bias of $2\delta$.
\end{definition}
This acts as a NAND gadget in the following sense.

\begin{lemma}
Let $\delta>0$ and $X$ be drawn from the MRF corresponding to $J_\delta$. Then $(X_1,X_2)$ takes on each possible pair of values with a probability in $[1/4-e^{-4\delta},1/4+e^{-4\delta}]$ and $\mathbb{P}\left[X_3\ne \left(X_1 \text{ AND } X_2\right)\right]\ge 1-e^{-4\delta}$.
\end{lemma}

\begin{proof}
Let $Z=4e^{3\delta}+3e^{-\delta}+e^{-9\delta}$. One can easily check that
\begin{align*}
\mathbb{P}[X=(-1,-1,-1)]&=e^{-9\delta}/Z\\
\mathbb{P}[X=(1,-1,-1)]&=e^{-\delta}/Z\\
\mathbb{P}[X=(-1,1,-1)]&=e^{-\delta}/Z\\
\mathbb{P}[X=(1,1,-1)]&=e^{3\delta}/Z\\
\mathbb{P}[X=(-1,-1,1)]&=e^{3\delta}/Z\\
\mathbb{P}[X=(1,-1,1)]&=e^{3\delta}/Z\\
\mathbb{P}[X=(-1,1,1)]&=e^{3\delta}/Z\\
\mathbb{P}[X=(1,1,1)]&=e^{-\delta}/Z
\end{align*}
The desired conclusion follows immediately.
\end{proof}

In particular, we can fuse $J_\delta$ with an existing MRF in order to add a vertex that takes on the value corresponding to the NAND of two other vertices with high probability. More formally, we have the following.

\begin{lemma}
Let $\delta>1$, $M$ be an MRF and $v$ and $u$ be two of its vertices. Next, let $M'$ be the MRF formed by taking $M$, and making the following changes. The biases of $v$ and $u$ are increased by $\delta$. If there was no edge between $v$ and $u$ in $M$ then there is an edge of weight $-\delta$ between them in $M'$, while if there was an edge between them it has a weight that is $\delta$ smaller in $M'$ than in $M$. Finally, $M'$ has a new vertex $v^\star$ which has a bias of $2\delta$, is connected to $v$ and $u$ by edges of weight $-2\delta$, and has no other edges. Now, let $X\sim M$ and $X'\sim M'$. Then the total variation distance between the probability distribution of $X$ and the probability distribution of the restriction of $X'$ to the vertices that $M$ has is at most $5e^{-4\delta}$ and $\mathbb{P}\left[X'_{v^\star}\ne \left(X'_v \text{ AND } X'_{u}\right)\right]\ge 1-5e^{-4\delta}$
\end{lemma}


\begin{proof}
First, observe that the biases and edge weights of $M'$ are the sum of the biases and edge weights in $M$ with the biases and edge weights of a copy of $J_\delta$ defined on the vertices $(v,u,v^\star)$. Now, let $X''$ be drawn from that copy of $J_\delta$. Next, let $Z$, $Z'$ and $Z''$ be the partition functions of $M$, $M'$, and $J_\delta$ respectively. It must be the case that for any possible value of $X'$, $x'$, 
\[Z'\cdot \mathbb{P}[X'=x']=\left (Z\cdot \mathbb{P}[X=x'_{-v^\star}]\right)\cdot \left(Z''\cdot \mathbb{P}[X''=x'_{(v,u,v^\star)}]\right)\]
That implies that for any $x$ it will be the case that
\begin{align*}
\mathbb{P}[X'_{-v^\star}=x]&=\frac{Z\cdot Z''}{Z'} \mathbb{P}[X=x]\cdot \left( \mathbb{P}[X''=(x_v,x_{u},1)]+\mathbb{P}[X''=(x_v,x_{u},-1)]\right)\\
&=\frac{Z\cdot Z''}{Z'} \mathbb{P}[X=x]\cdot \mathbb{P}[X''_{\{v,u\}}=x_{\{v,u\}}]\\
\end{align*}
and the fact that these probabilities must add up to $1$ implies that
\[\frac{Z'}{Z\cdot Z''}=\sum_{x'\in\{-1,1\}^2} \mathbb{P}[X_{\{v,u\}}=x']\cdot \mathbb{P}[X''_{\{v,u\}}=x']\]
Also, we know that $1/4-e^{-4\delta}\le  \mathbb{P}[X''_{\{v,u\}}=x_{\{v,u\}}]\le 1/4+e^{-4\delta}$ for all $x$, so $|\frac{Z'}{Z\cdot Z''}-1/4|\le e^{-4\delta}$ as well. That means that the total variation distance between the probability distributions of $X$ and $X'_{-v^\star}$ is
\begin{align*}
&\frac{1}{2}\sum_{x'\in\{-1,1\}^2} \left|\mathbb{P}[X_{\{v,u\}}=x']-\mathbb{P}[X'_{\{v,u\}}=x']\right|\\
&=\frac{1}{2}\sum_{x'\in\{-1,1\}^2} \mathbb{P}[X_{\{v,u\}}=x']\cdot \left|1-\frac{Z\cdot Z''}{Z'}\cdot \mathbb{P}[X''_{\{v,u\}}=x']\right|\\
&\le \frac{1}{2}\sum_{x'\in\{-1,1\}^2} \mathbb{P}[X_{\{v,u\}}=x']\cdot (10e^{-4\delta})\\
&\le 5e^{-4\delta}
\end{align*}
Furthermore, for any $x\in\{-1,1\}^2$, it must be the case that $\mathbb{P}[X'_{v^\star}=1 |X'_{\{v,u\}}=x]=\mathbb{P}[X''_{v^\star}=1|X''_{\{v,u\}}=x]$. We know that $\mathbb{P}\left[X''_{v^\star}= \left( X''_v \text{ AND }X''_{u}\right)\right]\le e^{-4\delta}$, and $X''_{\{v,u\}}$ takes on each possible value with probability at least $1/5$, so $\mathbb{P}[X'_{v^\star}\ne \left( X'_v \text{ AND } X'_{u}\right)]\ge 1-5e^{-4\delta}$ as desired.
\end{proof}


 It is well known that any polynomial sized circuit performing an arbitrary efficient computation can be implemented by a polynomial sized circuit consisting of NAND gates only, which implies the following.

\begin{lemma}
Let $A_n$ be an efficient randomized algorithm that samples from some probability distribution on $\{-1,1\}^n$. Then there exists a series of CMRFs, $M_n$, such that $M_n$ has $n$ visible vertices, a total number of vertices polynomial in $n$, and a probability distribution that is within a total variation distance of $O(e^{-n})$ from the probability distribution of the output of $A_n$.
\end{lemma}

\begin{proof}
First, observe that there must be a polynomial sized circuit made of NANDs that takes random bits as inputs and computes the output of $A_n$ when it receives those random values. So, we can start with an MRF that has the appropriate number of independent random elements, and then apply the modification given by the previous lemma with $\delta=-2n$ for every NAND in the circuit. That requires a polynomial number of applications of the modification procedure, and each application distorts the probability distribution of the previous vertices by $O(e^{-2n})$ and adds a new vertex that computes the appropriate NAND correctly with probability $1-O(e^{-2n})$, so the probability distribution of the resulting MRF is within a total variation distance of $O(e^{-n})$ of the probability distribution of the set of values taken by the gates and inputs of said circuit. So, if we censor all vertices in this MRF except the ones corresponding to the outputs, we get a CMRF with a probability distribution that is within $O(e^{-n})$ of that produced by the algorithm.
\end{proof}

\begin{remark}
We could have given some intermediate vertices very large biases in order to essentially force them to take on certain values. If we did so, then this would prove that we can create a CMRF that essentially generates a random certificate for an NP problem instance of our choice and then performs an efficient computation of our choice on it.

Among other things, that would allow us to create a CMRF that finds the algorithm that encodes to a given value with a given encryption scheme and public key, and then uses that algorithm to generate a sample. Assuming standard cryptographic assumptions are true, that means that knowing the parameters of such a CMRF tells us essentially nothing about how it behaves.
\end{remark}

\end{document}